\documentclass[a4paper,english,11pt]{article}
\usepackage[english]{babel}
\usepackage[utf8]{inputenc}
\usepackage[T1]{fontenc}
\usepackage{palatino}
\usepackage{url, hyperref}
\usepackage[svgnames]{xcolor}
\usepackage{amsmath}
\usepackage{amssymb}
\usepackage{amsthm}
\usepackage{bbm}
\usepackage{caption}
\usepackage{subcaption}
\usepackage{float}
\usepackage{tikz}

\definecolor{hrefcolor}{rgb}{0.0,0.5,0.8}
\definecolor{hlgreen}{rgb}{0,0.7,0}

\hypersetup{
   colorlinks=true,
   linkcolor=hrefcolor,
   citecolor=hlgreen,
   filecolor=hrefcolor,
   urlcolor=hrefcolor,
}

\usepackage{tikz, pgfplots}
\pgfplotsset{compat=newest}
\pgfplotsset{plot coordinates/math parser=false}
\usetikzlibrary{external}
\usepackage[textsize=footnotesize,colorinlistoftodos]{todonotes}

\newcommand{\field}[1]{\mathbb{#1}}
\newcommand{\R}{\field{R}}
\newcommand{\iprod}[2]{\langle #1,#2\rangle}
\newcommand{\extR}{\overline \R}
\newcommand{\subdiff}{\partial}
\newcommand{\defeq}{:=}
\DeclareMathOperator*{\argmin}{arg\,min}
\newcommand{\norm}[1]{\|#1\|}

\newcommand{\abs}[1]{|#1|}

\newcommand{\inv}[1]{#1^{-1}}
\newcommand{\grad}[1]{\nabla #1}
\renewcommand{\L}{\mathcal{L}}
\DeclareMathOperator{\BVspace}{BV}
\DeclareMathOperator{\BDspace}{BD}
\DeclareMathOperator{\TV}{TV}
\DeclareMathOperator{\TGV}{TGV}
\DeclareMathOperator{\Sym}{Sym}
\DeclareMathOperator{\divergence}{div}

\newcommand{\Tensor}{\mathcal{T}}
\newcommand{\Meas}{\mathcal{M}}
\let\d\relax
\DeclareMathOperator{\d}{d}
\DeclareMathOperator{\esssup}{ess\,sup}
\newcommand{\BALL}{V}
\newcommand{\mathvar}[1]{\textup{#1}}

\newcommand{\sym}{\mathvar{s}}
\newcommand{\freevar}{\,\boldsymbol\cdot\,}
\newcommand{\FA}{\mathvar{FA}}
\newcommand{\PSNR}{\mathvar{PSNR}}

\DeclareMathOperator{\lev}{lev}

\makeatletter
\def \weaktostar@sym{\setbox0=\hbox{$\rightharpoonup$}\rlap{\hbox
        to\wd0{\hss\raise1ex\hbox{$\scriptscriptstyle{*\,}$}\hss}}\box0}
    \def \weaktostar    {\mathrel{\weaktostar@sym}}
\makeatother

\newtheorem{theorem}{Theorem}

\newtheorem{lemma}{Lemma}
\newtheorem{proposition}{Proposition}

\theoremstyle{definition}
\newtheorem{definition}{Definition}

\newtheorem*{assumption*}{Assumption}
\newtheorem{remark}{Remark}
\newtheorem*{remark*}{Remark}
\newtheorem*{definition*}{Definition}
\newtheorem{example}{Example}

\numberwithin{equation}{section}
\numberwithin{lemma}{section}
\numberwithin{theorem}{section}
\numberwithin{proposition}{section}
\numberwithin{definition}{section}
\numberwithin{remark}{section}
\numberwithin{example}{section}
\numberwithin{assumption}{section}
\numberwithin{algorithm}{section}
\numberwithin{corollary}{section}

\newcommand{\term}[1]{\emph{#1}}
\newcommand*{\doi}[1]{doi:\href{http://dx.doi.org/#1}{\detokenize{#1}}}

\newlength{\w}

\title{Diffusion tensor imaging with deterministic error bounds}

\author{
    Artur Gorokh\thanks{Faculty of Physics, Lomonosov Moscow State University, Russia. \texttt{arturgorokh@yahoo.com}}
    \and
    Yury Korolev\thanks{School of Engineering and Materials Science,  Queen Mary University of London, United Kingdom. \texttt{korolev.msu@gmail.com}}
    \and
    Tuomo Valkonen\thanks{Department of Applied Mathematics and Theoretical Physics, University of Cambridge, United Kingdom. \texttt{tuomo.valkonen@iki.fi}}
    }

\makeatother
\hypersetup{
    pdftitle = {Diffusion tensor imaging with deterministic error bounds},
    pdfauthor = {Artur Gorokh, Yury Korolev, Tuomo Valkonen}
}
\makeatletter

\begin{document}

\maketitle

\begin{abstract}
    Errors in the data and the forward operator of an inverse problem can be handily modelled using partial order in Banach lattices. We present some existing results of the theory of regularisation in this novel framework, where errors are represented as bounds by means of the appropriate partial order.

    We apply the theory to diffusion tensor imaging (DTI), where correct noise modelling is challenging: it involves the Rician distribution and the nonlinear Stejskal-Tanner equation. Linearisation of the latter in the statistical framework would complicate the noise model even further. We avoid this using the error bounds approach, which preserves simple error structure under monotone transformations.
\end{abstract}

\section{Introduction}

Often in inverse problems,
we have only very rough knowledge of noise models, or the exact model is too difficult to realise in a numerical reconstruction method. The data may also contain process artefacts from black box devices \cite{pan2009why}.
Partial order in Banach lattices has therefore recently been investigated in~\cite{Kor_IP:2014,Kor_Yag_IP:2013,Kor_Yag_JIIP} as a less-assuming error modelling approach for inverse problems. 
This framework allows the representation of errors in the data as well as in the forward operator of an inverse problem by means of order intervals (i.e., lower and upper bounds by means of appropriate partial orders). An important advantage of this approach vs. statistical noise modelling is that deterministic error bounds preserve their simple structure under monotone transformations. 

We apply partial order in Banach lattices to diffusion tensor imaging (DTI).
We will in due course explain the diffusion tensor imaging progress, as well as the theory of inverse problems in Banach lattices, but start by introducing our model
\begin{equation}
    \notag
    \min_u~ R(u)
    \quad
    \text{subject to}
    \quad
    \begin{array}[t]{l}
        u \geqslant 0,\\
        g_j^l \leqslant A_j u \leqslant g_j^u,\quad \L^n\text{-a.e. on } \Omega,\ %
        (j=1,\ldots,N).
    \end{array}
\end{equation}
That is, we want to find a field of symmetric $2$-tensors $u: \Omega \to  \Sym^2(\R^3)$ on the domain $\Omega \subset \R^3$, minimising the value of the regulariser $R$ on the feasible set. The tensor field $u$ is our unknown image. It is subject a positivity constraint, as well as partial order constraints imposed through the operators $[A_j u](x) \defeq -\iprod{b_j}{u(x)b_j}$, and the upper and lower bounds $g_j^l \defeq \log(\hat s_j^l/\hat s_0^u)$ and $g_j^u \defeq \log(\hat s_j^u/\hat s_0^l)$. These model, in terms of error intervals after logarithmic transformation, the Stejskal--Tanner equation
\begin{equation}
    \label{eq:stejskal-tanner-intro}
    s_j(x)=s_0(x) \exp(-\iprod{b_j}{u(x)b_j}),
    \quad (j=1,\ldots,N),
\end{equation}
central to the diffusion tensor imaging process.

To shed more light on $u$ and the equation \eqref{eq:stejskal-tanner-intro}, let us briefly outline the diffusion tensor imaging process. As a first step towards DTI, diffusion weighted magnetic resonance imaging (DWI) is performed. This process measures the anisotropic diffusion of water molecules. To capture the diffusion information, the magnetic resonance images have to be measured with diffusion sensitising gradients in multiple directions. These are the different $b_i$'s in \eqref{eq:stejskal-tanner-intro}.
Eventually, multiple DWI images $\{s_j\}$ are related through the Stejskal-Tanner equation \eqref{eq:stejskal-tanner-intro} to the symmetric positive-definite {diffusion-tensor field} $u: \Omega \to \Sym^2(\R^3)$ \cite{basser2002diffusion,kingsley2006introduction}. At each point $x \in \Omega$, the tensor $u(x)$ is the covariance matrix of a normal distribution for the probability of water diffusing in different spatial directions.

\begin{figure}
    \centering
    \begin{subfigure}[t]{0.24\textwidth}
        \includegraphics[width=\textwidth]{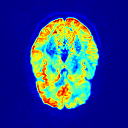}
        \caption{Slice of a real MRI measurement}
    \end{subfigure}
    \begin{subfigure}[t]{0.35\textwidth}
        \includegraphics[width=\textwidth]{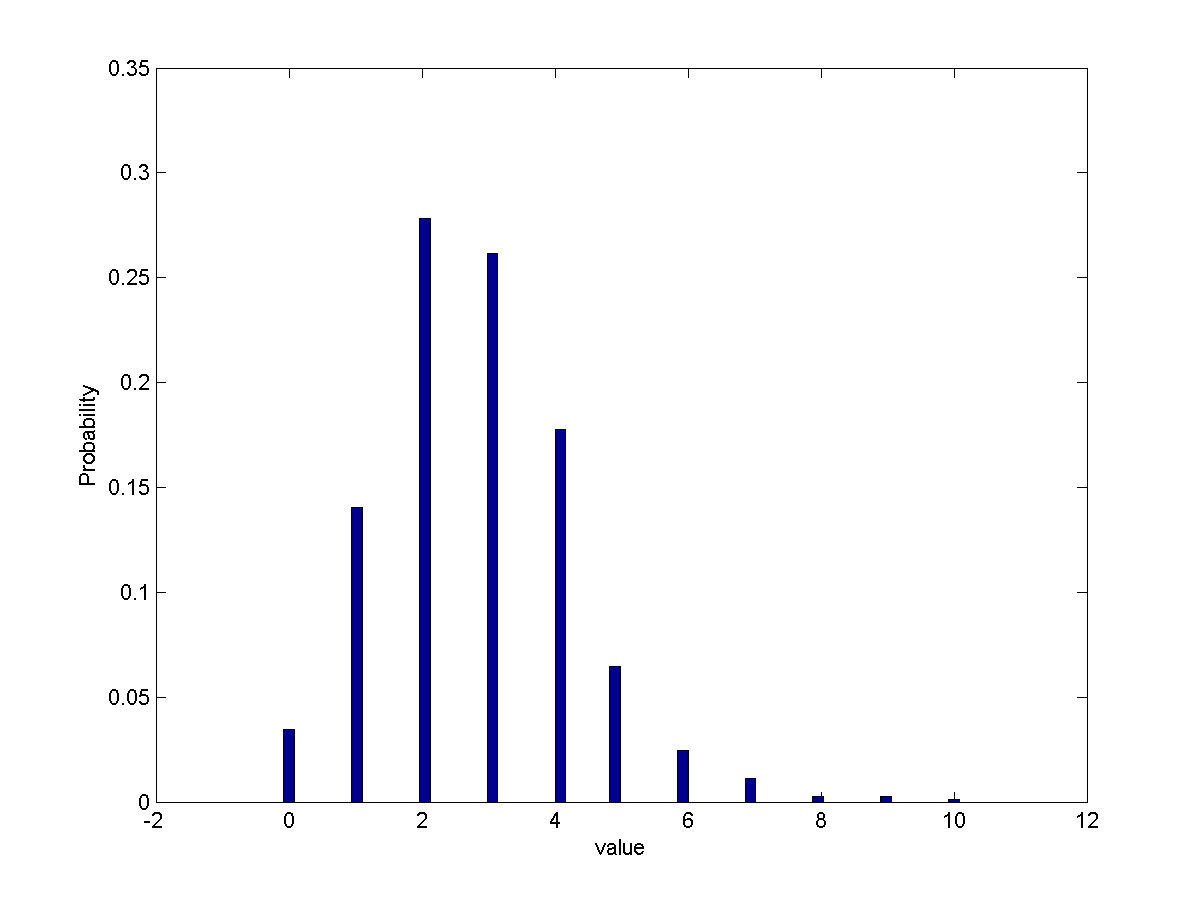}
        \caption{50-bin histogram of noise estimated from background}
    \end{subfigure}
    \begin{subfigure}[t]{0.35\textwidth}
        \includegraphics[width=\textwidth]{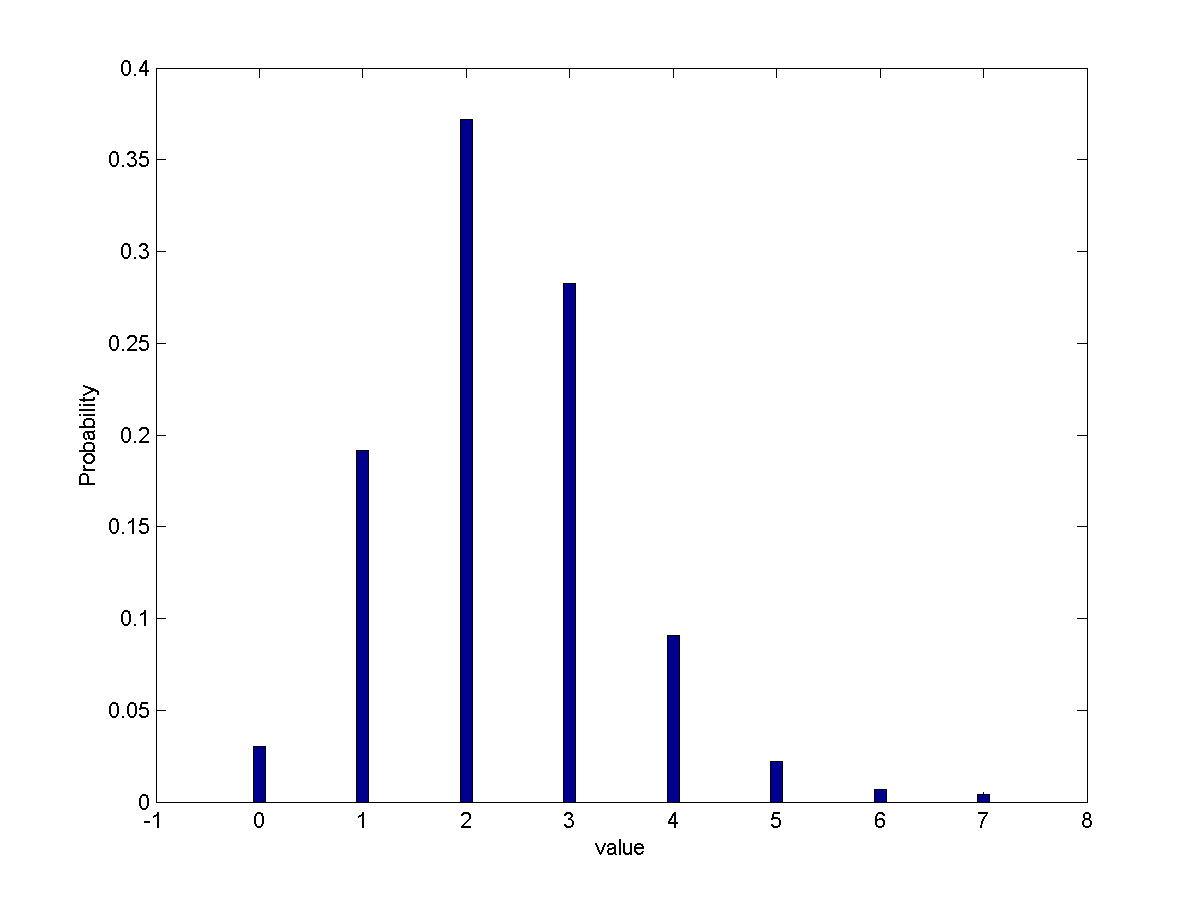}
        \caption{50-bin histogram after eddy-current correction with FSL}
    \end{subfigure}
    \caption{The noise in the absolute values of complex MRI data should be Rician. Here we have taken a 50-bin histogram of the noise in real data. This divides the pixels into bins of 50 different noise levels. However, we only find approximately 10 noise levels to have non-zero pixel count. As the Rician distribution is continuous, we see that the noise cannot be Rician, some bins of the 50-bin histogram being empty.
    The measurement setup of the data used here is described in Section~\ref{sec:invivo}.}
    \label{fig:noise-illustration}
\end{figure}

The fact that multiple $b_i$'s are needed to recover $u$, leads to very long acquisition times, even with ultra fast sequences like echo planar imaging (EPI). Therefore, DTI is inherently a low-resolution and low-SNR method. In theory, the amplitude DWI images exhibit Rician noise \cite{gudbjartsson1995rician}. However, as the histogram of an in vivo measurement in Figure \ref{fig:noise-illustration} illustrates, this may not be the case for practical data sets from black-box devices. Moreover, the DWI process is prone to eddy-current distortions \cite{tournier2011diffusion}, and due to the slowness of it, it is very sensitive to patient motion \cite{herbst2011prospective,aksoy2011real}. We therefore have to use techniques that remove these artefacts in solving for $u(x)$. We also need to ensure the positivity $u$, as non-positive-definite diffusion tensor are non-physical.
One proposed approach for the satisfaction of this constraint is that of log-Euclidean metrics \cite{arsigny2006log}.
This approach has several theoretically desirable aspects, but some practical shortcomings \cite{tuomov-dtireg}.
Special Perona-Malik type constructions on Riemannian manifolds can also be used to maintain the structure of the tensor field \cite{chefd2002constrained,tschumperle2001diffusion}.
Such anisotropic diffusion is however severely ill-posed \cite{weickert1998anisotropic}.
Recently manifold-valued discrete-domain total variation models have also been applied to diffusion tensor imaging \cite{bacak2015second}.

Our approach is also in the total variation family, first considered for diffusion tensor imaging in \cite{setzer2009variational}. Namely, we follow up on the work in \cite{tuomov-dtireg,escoproc,ipmsproc,tuomov-nlpdhgm} on the application of total generalised variation regularisation \cite{bredies2009tgv} to DTI. We should note that in all of these works the fidelity function was the ROF-type \cite{Rud1992} $L^2$ fidelity. This would only be correct, according to the assumption that noise of MRI measurements is Gaussian, if we had access to the original complex k-space MRI data. The noise of the inverse Fourier-transformed magnitude data $s_j$, that we have in practice access to, is however Rician under the Gaussian assumption on the original complex data \cite{gudbjartsson1995rician}. This is not modelled by the $L^2$ fidelity.

Numerical implementation of Rician noise modelling has been studied in \cite{martin2013tgvdti,getreuer2011rician}. As already discussed, in this work, we take the other direction. Instead of modelling the errors in a statistically accurate fashion, not assuming to know an exact noise model, we represent them by means of pointwise bounds.
The details of the model are presented in Section~\ref{sec:dti}. We study the practical performance in Section~\ref{sec:experiments} using the numerical method presented in Section~\ref{sec:optimisation}. First we however start with the general error modelling theory in Section \ref{sec:deterministic}.
Readers who are not familiar with notation for Banach lattices or symmetric tensor fields are advised to start with the Appendix, where we introduce our mathematical notation and techniques.



\section{Deterministic error modelling}
\label{sec:deterministic}

%

\subsection{Mathematical basis}

We now briefly outline the theoretical framework~\cite{Kor_IP:2014} that is the basis for our approach. Consider a linear operator equation
\begin{equation}\label{Az=u}
Au = f, \quad u \in U, \,\, f \in F,
\end{equation}
where $U$ and $F$ are Banach lattices, $A \colon U \to F$ is a regular injective operator. The inaccuracies in the right-hand side $f$ and the operator $A$ are represented as bounds by means of appropriate partial orders, i.e.
\begin{equation}\label{math_basis:bounds}
\begin{aligned}
&f^l, f^u  \colon &&f^l \leqslant_F f \leqslant_F f^u, \\
&A^l, A^u \colon &&A^l \leqslant_{L^\sim(U,F)} A \leqslant_{L^\sim(U,F)} A^u,
\end{aligned}
\end{equation}
\noindent where the symbol $\leqslant_F$ stands for the partial order in $F$ and $\leqslant_{L^\sim(U,F)}$ for the partial order for regular operators induced by partial orders in $U$ and $F$.  Further, we will drop the subscripts at inequality signs where it will not cause confusion.

The exact right-hand side $f$ and operator $A$ are not available. Given the approximate data $(f^l,f^u,A^l,A^u)$, we need to find an approximate solution $u$ that converges to the exact solution $\bar u$ as the inaccuracies in the data diminish. This statement needs to be formalised. We consider monotone convergent sequences of lower and upper bounds
\begin{equation}\label{math_basis:data}
\begin{aligned}
&f^l_n \colon f^l_{n+1} \geqslant f^l_n, &		\quad 	&A^l_n \colon A^l_{n+1} \geqslant A^l_n, \\	
&f^u_n \colon f^u_{n+1} \leqslant f^u_n, &	\quad	&A^u_n \colon A^u_{n+1} \leqslant A^u_n, \\
&f^l_n \leqslant f \leqslant f^u_n, 	      &	\quad 	&A^l_n \leqslant A \leqslant A^u_n  \quad \forall n \in \mathbb N,\\
&\| f^l_n - f^u_n \| \to 0, &				\quad	& \| A^l_n - A^u_n \| \to 0 \quad \text{as $n \to \infty$}.
\end{aligned}
\end{equation}
We are looking for an approximate solution $u_n$ such that $\|u_n-\bar u\|_U \to 0$ as $n \to \infty$.

Let us ask the following question. What are the elements $u \in U$ that could have produced data within the tolerances (\ref{math_basis:data})? Obviously, the exact solution is one of such elements. Let us call the set containing all such elements the feasible set $U_n \subset U$.

Suppose that we know \emph{a priori} that the exact solution is positive (by means of the appropriate partial order in $U$). Then it is easy to verify that the following inequalities hold for all $n \in \mathbb N$
\begin{equation*}
\bar u \geqslant_U 0, \quad A^u_n \bar u \geqslant_F f^l_n, \quad A^l_n \bar u \leqslant_F f^u_n.
\end{equation*}

This observation motivates our choice of the feasible set:
\begin{equation*}
U_n = \{ u \in U \colon u \geqslant_U 0, \quad A^u_n u \geqslant_F f^l_n, \quad A^l_n  u \leqslant_F f^u_n \}.
\end{equation*}

It is clear that the exact solution $\bar u$ belongs to the sets $U_n$ for all $n \in \mathbb N$. Our goal is to define a rule that will choose for any $n$ an element $u_n$ of the set $U_n$ such that the sequence $u_n \in U_n$ will strongly converge to the exact solution $\bar u$.
We do so by minimising an appropriate regularisation functional $R(u)$ on $U_n$:
\begin{equation}\label{u_n}
u_n = \argmin_{u \in U_n} R(u).
\end{equation}

This method, in fact, is a lattice analogue of the well-known residual method~\cite{GrasmairHalmeierScherzer2011,TGSYag}. The convergence result is as follows~\cite{Kor_IP:2014}.
\begin{theorem}\label{thm_strong}
Suppose that
\begin{enumerate}
\item $R(u)$ is bounded from below on $U$,
\item $R(u)$ is lower semi-continuous,
\item level sets $\{u \colon R(u) \leqslant C\}$ ($C=\emph{const}$) are sequentially compact in $U$ (in the strong topology induced by the norm).
\end{enumerate}
Then the sequence defined in~(\ref{u_n}) strongly converges to the exact solution $\bar u$ and $\mathcal R(u_n) \to \mathcal R(\bar u)$.
\end{theorem}

Examples of regularisation functionals that satisfy the conditions of Theorem~\ref{thm_strong} are as follows. Total Variation in $L^1(\Omega)$, where $\Omega$ is a subset of $\mathbb R^n$, assures strong convergence in $L^1$, given that the $L^1$-norm of the solution is bounded. The Sobolev norm $\| u \|_{W^{1,q}(\Omega)}$ yields strong convergence in the spaces $L^p(\Omega)$, where $p \geqslant 1$, $q > \frac{np}{p+n}$. The latter fact follows from the compact embedding of the corresponding Sobolev $W^{1,q}(\Omega)$ space into $L^p(\Omega)$~\cite{DS1}.

However, the assumption that the sets $\{u \colon R(u) \leqslant C\}$ are strong compacts in $U$ is quite strong. It can be replaced by the assumption of weak compactness, provided that the regularisation functional possesses the so-called Radon-Riesz property.

\begin{definition}
A functional $F \colon U \to \mathbb R$ has the Radon-Riesz property (sometimes referred to as the $H$-property), if for any sequence $u_n \in U$ weak convergence $u_n \rightharpoonup u_0$ and simultaneous convergence of the values $F(u_n) \to F(u_0)$ imply strong convergence $u_n \to u_0$.
\end{definition}

\begin{theorem}\label{thm_weak}
Suppose that
\begin{enumerate}
\item $R(u)$ is bounded from below on $U$,
\item $R(u)$ is weakly lower semi-continuous,
\item level sets $R(u) \leqslant C$ ($C=\emph{const}$) are weakly sequentially compact in $U$,
\item $R(u)$ possesses the Radon-Riesz property.
\end{enumerate}
Then the sequence defined in~(\ref{u_n}) strongly converges to the exact solution $\bar u$ and $\mathcal R(u_n) \to \mathcal R(\bar u)$.
\end{theorem}

It is easy to verify that the norm in any Hilbert space possesses the Radon-Riesz property. Moreover, this holds for the norm in any reflexive Banach space~\cite{DS1}.

As we explain in the Appendix, the spaces $L^p(\Omega; \Sym^{2}(\R^m))$ are not Banach lattices, therefore, Theorems~\ref{thm_strong} and~\ref{thm_weak} cannot be applied directly. Further theoretical work will be undertaken to extend the framework to the non-lattice case. 
For the moment, however, we will prove that if there are no errors in the operator $A$ in~(\ref{Az=u}), the requirement that the solution space $U$ is a lattice can be dropped.

\begin{theorem}\label{thm_fixed}
Let $U$ be a Banach space, and $F$ be a Banach lattice. Let the operator $A$ in~(\ref{Az=u}) be a linear, continuous and injective operator. Let $f^l_n$ and $f^u_n$ be sequences of lower and upper bounds for the right-hand side defined in~(\ref{math_basis:data}), and suppose that there are no errors in the operator $A$. Let us redefine the feasible set in the following way
\begin{equation*}
U_n = \{ u \in U \colon \quad f^l_n \leqslant_F A  u \leqslant_F f^u_n \}.
\end{equation*}
Suppose also that the regulariser $R(x)$ satisfies conditions of either Theorem~\ref{thm_strong} or Theorem~\ref{thm_weak}. Then the sequence defined in~(\ref{u_n}) strongly converges to the exact solution $\bar u$  and $R(u_n) \to R(\bar u)$.
\end{theorem}
\begin{proof}
Let us define an approximate right-hand side and its approximation error in the following way
\begin{equation*}
f_{\delta_n} = \frac{f^u_n + f^l_n}{2}, \quad \delta_n = \frac{\|f^u_n-f^l_n\|}{2}.
\end{equation*}
One can easily verify, that the inequality $\|f-f_{\delta_n}\| \leqslant \delta_n$ holds. Indeed, we have
\begin{equation*}
\begin{aligned}
&f - f_{\delta_n} \leqslant f^u_n - f_{\delta_n} = \frac{f^u_n-f^l_n}{2}, \\
&-(f - f_{\delta_n}) \leqslant f_{\delta_n} - f^l_n = \frac{f^u_n-f^l_n}{2},\\
&|f - f_{\delta_n}| = (f - f_{\delta_n}) \vee (-(f - f_{\delta_n})) \leqslant \frac{f^u_n-f^l_n}{2},\\
&\|f - f_{\delta_n}\| \leqslant \frac{\|f^u_n-f^l_n\|}{2}.
\end{aligned}
\end{equation*}
The first two inequalities are consequences of the conditions~(\ref{math_basis:data}), the third one holds by the definition of supremum and the equality $|f| = f \vee (-f)$ that holds for all $f \in F$, and the last inequality is due to the monotonicity of the norm in a Banach lattice.

Similarly, one can show that for any $u \in U_n$, we have
\begin{equation*}
\|Au - f_{\delta_n}\| \leqslant \delta_n.
\end{equation*}
Therefore, the inclusion $U_n \subset \{u \colon \|Au-f_{\delta_n}\| \leqslant \delta_n\}$ holds.

Now we will proceed with the proof of convergence $\|u_n-\bar u\| \to 0$. Will prove it for the case when the regulariser $R(u)$ satisfies conditions of Theorem~\ref{thm_strong}. Suppose that the sequence $u_n$ does not converge to the exact solution $\bar u$. Then it contains a subsequence $u_{n_k}$ such that $\|u_{n_k} - \bar u\| \geqslant \varepsilon$ for any $k \in \mathbb N$ and some fixed $\varepsilon >0$.

Since the inclusion $\bar u \in U_n$ holds for all $n \in \mathbb N$, we have $R(u_n) \leqslant R(\bar u)$ for all $n \in \mathbb N$. Since the level set $\{u \colon R(u) \leqslant R(\bar u)\}$ is a compact set by assumptions of the theorem, the sequence $u_{n_k}$ contains a strongly convergent subsequence. With no loss of generality, let us assume that $u_{n_k} \to u_0$. We will now show that $u_0 = \bar u$. Indeed, we have
\begin{eqnarray*}
\|A u_{n_k} - A \bar u\| \leqslant \|A u_{n_k} - f_{\delta_n}\| + \|f-f_{\delta_n}\| \leqslant 2\delta_{n_k} \to 0.
\end{eqnarray*}
On the other hand, we have
\begin{equation*}
\|A u_{n_k} - A \bar u\| \to \|A u_0 - A \bar u \|
\end{equation*}
due to continuity of $A$ and $\| \cdot \|$. Therefore, $A u_0 = A \bar u$ and $u_0 = \bar u$, since $A$ is an injective operator. By contradiction, we get $\|u_n - \bar u\| \to 0$.

Finally, since the regulariser $R(u)$ is lower semi-continuous, we get that $\lim\inf R(u_n) = R(\bar u)$. However, for any $n$ we have $R(u_n) \leqslant R(\bar u)$, therefore, we get the convergence $R(u_n) \to R(\bar u)$ as $n \to \infty$.
\end{proof}

\subsection{Philosophical discussion and statistical interpretation}
\label{sec:confidence}

In practice, our data is discrete. So let us momentarily switch to measurements $\hat f = (\hat f^1, \ldots, \hat f^n) \in \R^n$ of a true data $f \in \R^n$. If we actually knew the pointwise noise model of the data, then one way to obtain potentially useful upper and lower bounds for the deterministic model is by means of statistical interval estimates: confidence intervals. Roughly, the idea is to find for each true signal $f$ individual \emph{random} upper and lower bounds $\hat f^u$ and $\hat f^l$ such that
\[
    P(\hat f^u \le f \le \hat f^l) = 1-\theta.
\]
If $\hat f^u$ and $\hat f^l$ are computed based on multiple experiments (i.e., multiple noisy samples $\hat f_1,\dots,\hat f_m$, of the true data $f$), the interval $[\hat f^{u,i}, \hat f^{l,i}]$ will converge in probability to the true data $\hat f^i$, as the number of experiments $m$ increases. Thus we obtain a probabilistic version of the convergences in \eqref{math_basis:data}.

Let us try to see, how such intervals might work in practice. For the purpose of the present discussion, assume that the noise is additive and normal-distributed with variance $\sigma$ and zero mean---an assumption that does not hold in practice, as we have already seen in Figure \ref{fig:noise-illustration}, but will suffice for the next thought experiments.
That is, $\hat f_j=f+\nu_j$ for the noise $\nu_j$. Let the sample mean of $\{\hat f_j\}_{j=1}^m$ be $\bar f_m = (\bar f^1_m, \ldots \bar f^n_m)$, and pointwise sample variance $\sigma=(\sigma^1,\ldots,\sigma^n)$. The product of the pointwise confidence intervals $I_i$ with confidence $1-\theta$ is \cite{cox1979theoretical,shiryaev1996probability}
\[
    \prod_{i=1}^n I_i = [\bar f_m - k^*_{1-\theta/2}\frac{\sigma}{\sqrt{m}}, \bar f_m +k^*_{1-\theta/2}\frac{\sigma}{\sqrt{m}}],
    \quad
    k^*_{1-t} \defeq \inv\Phi(t),
\]
for $\Phi$ the cumulative normal distribution function. For $\theta=0.05$, i.e., the $95\%$ confidence interval, $\inv\Phi(0.05/2)=1.96$. Now, the probability that $I_i$ covers the true  $f^i$ is $1-\theta$, e.g. $95\%$. If we have only a single sample $m=1$, the intervals stay large, but the joint probability, $(1-\theta)^n$ goes to zero as $n$ increases.
As an example, for a rather typical single $128 \times 128$ slice of a DTI measurement, the probability that exactly $\phi=5\%$ (to the closest discrete value possible) 
of the $1-\theta=95\%$ confidence intervals do not cover the true parameter would be about $1.4\%$, or
\[
    1.4\% \approx {n \choose m} \theta^m {(1-\theta)}^{n-m},
    \quad
    \text{where } n=128^2 \text{ and } m=\lceil \phi n \rceil.
\]
The probability of \emph{at least} $\phi=5\%$ of the pointwise $95\%$ confidence intervals not covering the true parameter is in this setting approximately $49\%$.
This can be verified by summing the above estimates over $m=\lceil \phi n \rceil,\ldots,n$.
%
%

In summary, unless $\theta$ simultaneously goes to $1$, the product intervals are very unlikely to cover the true parameter. Based on a single experiment, the deterministic approach as interpreted statistically through confidence intervals, is therefore very likely to fail to discover the true solution as the data size $n$ increases unless the pointwise confidence is very low. But, if we let the pointwise confidences be arbitrarily high, such that the intervals are very large, the discovered solution in our applications of interest would be just a constant!

Asymptotically, the situation is more encouraging.
Indeed, if we could perform more experiments to compute the confidence intervals, then for any fixed $n$ and $\theta$, it is easy to see that the solution of the ``deterministic'' error model is an asymptotically consistent and hence asymptotically unbiased estimator of the true $f$. That is, the estimates converge in probability to $f$ as the experiment count $m$ increases. Indeed, the error-bounds based estimator $\tilde f_m$, based on $m$ experiments, by definition satisfies $\tilde f_m \in \prod_{i=1}^n I_i$. Therefore, we have
\[
    P(\abs{\tilde f^i_m-\bar f_m^i}>\epsilon \text{ for some } i) = 0
    \quad
    \text{whenever}
    \quad
    m \ge (k^*_{1-\theta/2}\sigma/\epsilon)^2.
\]
Thus $f \overset{P}{\to} \bar f$ in probability. Since by the law of large numbers also $\bar f_m \overset{P}{\to} f$, this proves the claim, and to some extent justifies our approach from the statistical viewpoint.

It should be noted that this is roughly the most that has previously been known of the maximum a posteriori estimate (MAP), corresponding to the Tikhonov models
\[
    \min \frac{1}{2}\norm{\hat f-u}^2 + \alpha R(u).
\]
In particular, the MAP is not the Bayes estimator for the typical squared cost functional. This means that it does not minimise $\tilde f \mapsto \mathbb{E}[\norm{f-\tilde f}^2]$. The minimiser in this case is the conditional mean (CM) estimate, which is why it has been preferred by Bayesian statisticians despite its increased computational cost. The MAP estimate is merely an asymptotic Bayes estimator for the uniform cost function.
In a very recent work \cite{burger2014maximum}, it has however been proved that the MAP estimate \emph{is} the Bayes estimator for certain Bregman distances. One possible critique of the result is that these distances are not universal and do depend on the regulariser $R$, unlike the squared distance for CM. The CM estimate however has other problems in the setting of total variation and its discretisation \cite{lassas2004can,lassas2009discretization}.

\section{Application to diffusion tensor imaging}
\label{sec:dti}

We now build our model for applying the deterministic error modelling theory to diffusion tensor imaging. We start by building our forward model based on the Stejskal-Tanner equation, and then briefly introduce the regularisers we use.

\subsection{The forward model}

For $u: \Omega \to \Sym^2(\R^3)$, $\Omega \subset \R^3$, a mapping from $\Omega$ to symmetric second order tensors, let us introduce non-linear operators $T_j$, defined by
\[
    [T_j(u)](x) \defeq s_0(x) \exp(-\iprod{b_j}{u(x)b_j}),
    \quad
    (j=1,\ldots,N).
\]
Their role is to model the so-called Stejskal-Tanner equation
\cite{basser2002diffusion}
\begin{equation}
    \label{eq:stejskal-tanner}
    s_j(x)=s_0(x) \exp(-\iprod{b_j}{u(x)b_j}),
    \quad (j=1,\ldots,N).
\end{equation}
Each tensor $u(x)$ models the covariance of a Gaussian probability distribution at $x$ for the diffusion of water molecules. The data $s_j \in L^2(\Omega)$, ($j=1,\ldots,N$), are the diffusion-weighted MRI images. Each of them is obtained by performing the MRI scan with a different non-zero diffusion sensitising gradient $b_j$, while $s_0$ is obtained with a zero gradient. After correcting the original $k$-space data for coil sensitivities, each $s_j$ is assumed real. As a consequence, any measurement $\hat s_j$ of $s_j$ has---in theory---Rician noise distribution \cite{gudbjartsson1995rician}.

Our goal is to reconstruct $u$ with simultaneous denoising. Following \cite{tuomov-nlpdhgm,knoll2015model}, we consider using a suitable regulariser $R$ the Tikhonov model
\begin{equation}
    \label{eq:dti-recons-nl}
    \min_{u \geqq 0}~ \sum_{j=1}^N \frac{1}{2} \norm{\hat s_j-T_j(u)}^2 + \alpha R(u).
\end{equation}
The constraint $u\geqq 0$ is to be understood in the sense that $u(x)$ is positive semidefinite for $\L^n$-a.e.~$x \in \Omega$ (see Appendix for more details).
Due to the Rician noise of $\hat s_j$, the Gaussian noise model implied by the $L^2$-norm in \eqref{eq:dti-recons-nl} is not entirely correct.
However, in some cases the $L^2$ model may be accurate enough, as for suitable parameters the Rician distribution is not too far from a Gaussian distribution.
If one were to model the problem correctly, one should either modify the fidelity term to model Rician noise, or include the (unit magnitude complex number) coil sensitivities in the model. The Rician noise model is highly nonlinear due to the Bessel functional logarithms involved. Its approximations have been studied in \cite{basu2006medical,getreuer2011rician,martin2013tgvdti} for single MR images and DTI. Coil sensitivities could be included either by knowing them in advance, or by simultaneous estimation as in \cite{knoll2012parallel}. Either way, significant complexity is introduced into the model, and for the present work, we are content with the simple $L^2$ model.

We may also consider, as is often the case, and as was done with TGV in \cite{tuomov-dtireg}, the linearised model
\begin{equation}
    \label{eq:dti-recons-lin1}
    \min_{u \geqq 0}~ 
    \norm{f-u}^2 + \alpha R(u),
\end{equation}
where, for each $x \in \Omega$, $f(x)$ is solved by regression for $u(x)$ from the system of equations \eqref{eq:stejskal-tanner} with $s_j(x)=\hat s_j(x)$. Further, as in \cite{ipmsproc}, we may also consider
\begin{equation}
    \label{eq:dti-recons-lin2}
    \min_{u \geqq 0}~ \sum_{j=1}^N \frac{1}{2} \norm{g_j-A_j u}^2 + \alpha R(u),
\end{equation}
with $[A_j u](x) \defeq -\iprod{b_j}{u(x)b_j}$, and $g_j(x) \defeq \log(\hat s_j(x)/\hat s_0(x))$. In both of these linearised models, the assumption of Gaussian noise is in principle even more remote from the truth than in the nonlinear model \eqref{eq:dti-recons-nl}.
We will employ \eqref{eq:dti-recons-lin1} and \eqref{eq:dti-recons-nl} as benchmark models.

We want to further simplify the model, and forgo with accurate noise modelling. After all, we often do not know the real noise model for the data available in practice. It can be corrupted by process artefacts from black-box algorithms in the MRI devices. This problem of black box devices has been discussed extensively in \cite{pan2009why}, in the context of Computed Tomography. Moreover, as we have discussed above, even without such artefacts, the correct model may be difficult to realise numerically. 
So we might be best off choosing the least assuming model of all -- that of error bounds. This is what we propose in the reconstruction model
\begin{equation}
    \label{eq:dti-recons-constr}
    \min_u~ R(u)
    \quad
    \text{subject to}
    \quad
    \begin{array}[t]{l}
        u \geqq 0,\\
        g_j^l \leqslant A_j u \leqslant g_j^u,\quad \L^n\text{-a.e.},\ %
        (j=1,\ldots,N).
    \end{array}
\end{equation}
Here $g_j^l \defeq \log(\hat s_j^l/\hat s_0^u)$ and $g_j^u \defeq \log(\hat s_j^u/\hat s_0^l)$, $g_j^l, g_j^u \in L^2(\Omega)$, are our upper and lower bounds on $g_j$ that we derive from the data.

\subsection{Choice of the regulariser $R$}

A prototypical regulariser in image processing is the total variation, first studied in this context in \cite{Rud1992}. It can be defined for a symmetric tensor field $u \in L^1(\Omega; \Sym^k(\R^m))$ as
\[
    \begin{split}
    \TV(u)
    & \defeq \norm{Eu}_{\Meas(\Omega; \Sym^{k+1}(\R^m))}
    \\
    & \defeq \sup\left\{
        \int_\Omega \iprod{\divergence \phi(x)}{u(x)} \d x
        \middle|
        \begin{array}{l}
        \phi \in C_c^\infty(\Omega; \Sym^{k+1}(\R^m))
        \\
        \sup_x \norm{\phi(x)}_F \le 1
        \end{array}
        \right\}.
    \end{split}
\]
Observe that for scalar or vector fields, i.e., the cases $k=0,1$, we have $\Sym^0(\R^m)=\Tensor^0(\R^m)=\R$, and $\Sym^1(\R^m)=\Tensor^1(\R^m)=\R^m$. Therefore, for scalars in particular, this gives the usual isotropic total variation
\[
    \TV(u) = \norm{Du}_{\Meas(\Omega))}.
\]

%
%

Total generalised variation was introduced in \cite{bredies2009tgv} as a higher-order extension of $\TV$. Following \cite{tuomov-dtireg}, the second-order variant may be defined using the differentiation cascade formulation for symmetric tensor fields $u \in L^1(\Omega; \Sym^{k}(\R^m))$ as the marginal
\begin{equation}
    \label{eq:tgv2-cascade}
    \TGV^2_{(\beta,\alpha)}(u) \defeq
        \min\{
            \Phi_{(\beta,\alpha)}(u, w)
            \mid
            w \in L^1(\Omega; \Sym^{k+1}(\R^m))
            \}
\end{equation}
for
\[
    \Phi_{(\beta,\alpha)}(u, w) \defeq
            \alpha \norm{E u - w}_{F,\Meas(\Omega; \Sym^{k+1}(\R^m))}
            \\ \phantom{\Bigl(}
            +\beta \norm{E w}_{F,\Meas(\Omega; \Sym^{k+2}(\R^m))}.
\]

It turns out that the standard BV norm
\[
    \norm{u}_{\BVspace(\Omega; \Sym^k(\R^m))}
    \defeq \norm{u}_{L^1(\Omega; \Sym^k(\R^m))}+\TV(u)
\]
and the ``BGV norm'' \cite{bredies2009tgv}
\[
    \norm{u}' \defeq \norm{u}_{L^1(\Omega; \Sym^k(\R^m))}+\TGV^2_{(\beta,\alpha)}(u)
\]
are topologically equivalent norms \cite{l1tgv,sampta2011tgv} on $\BVspace(\Omega; \Sym^k(\R^m))$, yielding the same convergence results for TGV regularisation as for TV regularisation. 
The geometrical 
regularisation behaviour is however different, and TGV tends to avoid the staircasing observed in TV regularisation.

Regarding topologies, we say that a sequence $\{u^i\}$ in $\BVspace(\Omega; \Sym^k(\R^m))$ converges \term{weakly*} to $u$, if $u^i \to u$ strongly in $L^1$, and $Eu^i \weaktostar Eu$ weakly* as Radon measures \cite{ambrosio2000fbv,temam1985mpp,tuomov-dtireg}. The latter is characterised as $\int_\Omega \iprod{\divergence \phi(x)}{u^i(x)} \d x \to  \int_\Omega \iprod{\divergence \phi(x)}{u(x)} \d x$ for all $\phi \in C_c^\infty(\Omega; \Sym^{k+1}(\R^m))$.

\subsection{Compact subspaces}

Now, for a weak* lower semi-continuous seminorm $R$ on $\BVspace(\Omega; \Sym^k(\R^m))$, let us set
\[
    \BVspace_{0,R}(\Omega; \Sym^k(\R^m))
    \defeq
    \BVspace(\Omega; \Sym^k(\R^m)) / \ker R.
\]
That is, we identify elements $u, \tilde u \in \BVspace(\Omega; \Sym^k(\R^m))$, such that
$R(u - \tilde u)=0$. 
Now $R$ is a norm on $\BVspace_{0,R}(\Omega; \Sym^k(\R^m))$; compare, e.g., \cite{meyer2002oscillating} for the case of $R=\TV$. 

Suppose
\[
    \norm{u}' \defeq \norm{u}_{L^1(\Omega)} + R(u)
\]
is a norm on $\BVspace(\Omega; \Sym^k(\R^m))$, equivalent to the standard norm.
If also the \emph{$R$-Sobolev-Korn-Poincaré inequality}
\begin{equation}
    \label{eq:r-poincare}
    \inf_{R(v)=0} \norm{u-v}_{L^1(\Omega)} \le C R(u)
\end{equation}
holds, we may then bound
\begin{multline}
    \notag
    \inf_{R(v)=0} \norm{u-v}_{\BVspace(\Omega; \Sym^k(\R^m))}
    \le
    \inf_{R(v)=0} C' \norm{u-v}'
    \\
    =
    \inf_{R(v)=0} C' \bigl( \norm{u-v}_{L^1(\Omega)} + R(u-v)\bigr)
    \le C' (1+C)R(u).
\end{multline}
By the weak* lower semicontinuity of the BV-norm, and the weak* compactness of the unit ball in $\BVspace(\Omega; \Sym^k(\R^m))$---we refer to \cite{ambrosio2000fbv} for these and other basic properties of BV-spaces---we may thus find a representative $\tilde u$ in the $\BVspace_{0,R}(\Omega; \Sym^k(\R^m))$ equivalence class of $u$, satisfying
\[
    \norm{\tilde u}_{\BVspace(\Omega; \Sym^k(\R^m))} \le C' (1+C)R(u).
\]
Again using the weak* compactness of the unit ball in $\BVspace(\Omega; \Sym^k(\R^m))$, and the weak* lower semicontinuity of $R$, it follows that the sets
\[
    \lev_a R \defeq \{ u \in \BVspace_{0,R}(\Omega; \Sym^k(\R^m)) \mid R(u) \le a \},
    \quad (a>0),
\]
are weak* compact in $\BVspace_{0,R}(\Omega; \Sym^k(\R^m))$, in the topology inherited form $\BVspace(\Omega; \Sym^k(\R^m))$. Consequently, they are strongly compact subsets of $L^1(\Omega; \Sym^k(\R^m))$. This feature is crucial for the application of the regularisation theory in Banach lattices above.

On a connected domain $\Omega$, in particular
\[
    \BVspace_{0,\TV}(\Omega) \simeq \left\{ u \in \BVspace(\Omega) \middle| \int_\Omega u \d x=0 \right\}.
\]
That is, the space consists of zero-mean functions. Then $u \mapsto \norm{Du}_{\Meas(\Omega; \R^m)}$ is a norm on $\BVspace_{0,\TV}(\Omega)$ \cite{meyer2002oscillating}, and this space is weak* compact. In particular, the sets $\lev_a \TV$ are compact in $L^1(\Omega)$.

More generally, we know from \cite{bredies2014symmetric} that on a connected domain $\Omega$, $\ker \TV$ consists of $\Sym^k(\R^m)$-valued polynomials of maximal degree $k$. By extension, $\ker \TGV^2$ consists of $\Sym^k(\R^m)$-valued polynomials of maximal degree $k+1$. In both cases, \eqref{eq:r-poincare}, weak* lower semicontinuity of $R$, and the equivalence of $\norm{\freevar}'$ to $\norm{\freevar}_{\BVspace(\Omega; \Sym^k(\R^m))}$ hold by the results in \cite{bredies2014symmetric,sampta2011tgv,temam1985mpp}. Therefore, we have proved the following.

\begin{lemma}
    Let $\Omega \subset \R^m$ and $k \ge 0$. Then the sets $\lev_a \TV$ and $\lev_a \TGV^2$ are weak* compact in $\BVspace(\Omega; \Sym^k(\R^m))$ and strongly compact in $L^1(\Omega; \Sym^k(\R^m))$.
\end{lemma}

Now, in the above cases, $\ker R$ is finite-dimensional, and we may write
\[
    \BVspace(\Omega; \Sym^k(\R^m))
    \simeq \BVspace_{0,R}(\Omega; \Sym^k(\R^m)) \oplus \ker R.
\]
Denoting by
\[
    B_X(r) \defeq \{x \in X \mid \norm{x} \le r\},
\]
the closed ball of radius $r$ in a normed space $X$, we obtain by the finite-dimensionality of $\ker R$ the following result.

\begin{proposition}
    \label{prop:R}
    Let $\Omega \subset \R^m$ and $k \ge 0$. Pick $a>0$.
    Then the sets
    \[
        V \defeq \lev_a R \oplus B_{\ker R}(a)
    \]
    for $R=\TV$ and $R=\TGV^2$ are weak* compact in $\BVspace(\Omega; \Sym^k(\R^m))$ and strongly compact in $L^1(\Omega; \Sym^k(\R^m))$.
\end{proposition}

The next result summarises Theorem~\ref{thm_fixed} and Proposition~\ref{prop:R}.

\begin{theorem}
    With $U=L^1(\Omega; \Sym^k(\R^m))$,
    let the operator $A: U \to F$ be linear, continuous and injective. Let $f^l_n$ and $f^u_n$ be sequences of lower and upper bounds for the right-hand such that 
    \begin{equation*}
\begin{aligned}
&f^l_n \colon f^l_{n+1} \geqslant f^l_n, &		\quad 	&f^u_n \colon f^u_{n+1} \leqslant f^u_n, \\	
&f^l_n \leqslant f \leqslant f^u_n, 	      &	\quad 	&\| f^l_n - f^u_n \| \to 0 \quad \text{as $n \to \infty$}.
\end{aligned}
\end{equation*}
    \noindent Supposing that there are no errors in the operator $A$ and the exact solution $\bar u$ exists, define the feasible set as follows
    \begin{equation*}
        U_n = \{ u \in U \colon \quad f^l_n \leqslant_F A  u \leqslant_F f^u_n \}.
    \end{equation*}
    Decomposing $u \in U$ as $u=u_0+u^\perp$ with $u^\perp \in \ker R$, suppose 
    \begin{equation}
        \label{eq:u-perp-bound}
        u \in U_n \implies \norm{u^\perp} \le a
    \end{equation}
    for some constant $a>0$, then for $R=\TV$ and $R=\TGV^2$, the sequence 
    \begin{equation*}
    u_n = \argmin_{u \in U_n} R(u)
    \end{equation*}
    \noindent converges strongly in $L^1(\Omega; \Sym^k(\R^m))$ to the exact solution $\bar u$ and $R(u_n) \to R(\bar u)$.
\end{theorem}
\begin{proof}
    With the decomposition $u_n = u_{0,n} + u_n^\perp$, where $u_n^\perp \in \ker R$,
    we have $u_{0,n} \in \lev_a R$ for suitably large $a>0$ through
    \[
        R(u_{0,n}) = R(u_n) = \min_{u' \in U_n} R(u') \le R(\bar u).
    \]
    The assumption \eqref{eq:u-perp-bound} bounds $\norm{u_n^\perp} \le a$. Thus $u_n \in V$ for $V$ as in Proposition \ref{prop:R}. The proposition thus implies the necessary compactness in $U=L^1(\Omega; \Sym^k(\R^m))$ for the application of Theorem \ref{thm_fixed}.
\end{proof}

\begin{remark}
    The condition \eqref{eq:u-perp-bound} simply says for $R=\TV$ that the data has to bound the solution in mean. This is very reasonable to expect for practical data; anything else would be very non-degenerate. For $R=\TGV^2$ we also need that the data bounds the entire affine part of the solution. Again, this is very likely for real data.
    Indeed, in DTI practice, with at least 6 independent diffusion sensiting gradients, $A$ is an invertible or even over-determined linear operator. In that typical case, the bounds $f^l_n$ and $f^u_n$ will be translated into $U_n$ being a bounded set.
\end{remark}

\section{Solving the optimisation problem}
\label{sec:optimisation}

\subsection{The Chambolle--Pock method}

The Chambolle--Pock algorithm is an inertial primal-dual back\-ward-back\-ward
splitting method, classified in \cite{esser2010general} as the modified primal-dual hybrid
gradient method (PDHGM). It solves the minimax problem
\begin{equation}
    \label{eq:saddle}
    \min_x \max_y\ G(x) + \iprod{Kx}{y} - F^*(y),
\end{equation}
where $G: X \to \extR$ and $F^*: Y \to \extR$ are convex, proper, lower semicontinuous functionals on (finite-dimensional) Hilbert spaces $X$ and $Y$. The operator $K\colon X \to Y$ is linear, although an extension of the method to nonlinear $K$ has recently been derived \cite{tuomov-nlpdhgm}. The PDHGM can also be seen as a preconditioned ADMM (alternating
directions method of multipliers); we refer to \cite{esser2010general,setzer2011operator,tuomov-big-images} for reviews of optimisation methods popular in image processing.
For step sizes $\tau,\sigma>0$, and an over-relaxation parameter $\omega>0$, each iteration of the algorithm consists of the updates
\begin{subequations}
    \label{eq:cp}
    \begin{align}
        u_{i+1} & \defeq (I+\tau \subdiff G)^{-1}(u_i - \tau K^* y_{i}),
        \\
        \bar u_{i+1} & \defeq u_{i+1} + \omega(u_{i+1}-u_i),
        \\
        y_{i+1} & \defeq (I+\sigma \subdiff F^*)^{-1}(y_i + \sigma K \bar u_{i+1}).
    \end{align}
\end{subequations}
We should remark that the order of the primal ($u$) and dual ($y$) updates here is reversed from the original presentation in \cite{chambolle2010first}. The reason is that when reordered, the updates can, as discovered in \cite{he2012convergence}, be easily written in a proximal point form.

The first and last update are the backward (proximal) steps for the primal ($x$) and dual ($y$) variables, respectively, keeping the other fixed. However, the dual step includes some ``inertia'' or over-relaxation, as specified by the parameter $\omega$. Usually $\omega=1$, which is required for convergence proofs of the method. If $G$ or $F^*$ is uniformly convex, by smartly choosing for each iteration the step length  parameters $\tau,\sigma$, and the inertia $\omega$, the method can be shown to have convergence rate $O(1/N^2)$. This is similar to Nesterov's
optimal gradient method \cite{nesterov1983method}. In the general case the rate is $O(1/N)$. In practice the method produces visually pleasing solutions in rather few iterations, when applied to image processing problems.

In implementation of the method, it is crucial that the resolvents $(I+\tau \subdiff G)^{-1}$ and $(I+\sigma \subdiff F^*)^{-1}$ can be computed quickly. We recall that they may be written
as
\[
    (I+\tau \subdiff G)^{-1}(u)
    = \argmin_{u'}\left\{\frac{\norm{u'-u}^2}{2\tau} + G(u')\right\}.
\]
Usually in applications, these computations turn out to be simple projections or linear operations -- or the soft-thresholding operation for the $L_1$-norm.

As a further implementation note, since the algorithm \eqref{eq:cp} is formulated in Hilbert spaces (see however \cite{hohage2014generalization}), and we work in the Banach space $\BVspace(\Omega; \Sym^2(\R^3))$, we have to discretise our problems before application of the algorithm. We do this by simple forward-differences discretisation of the operator $E$ with cell width $h=1$ on a regular rectangular grid corresponding to the image voxels.

\subsection{Implementation of deterministic constraints}

We now reformulate the problem \eqref{eq:dti-recons-constr} of DTI imaging with deterministic error bounds in the form \eqref{eq:saddle}. Suppose we have some upper and lower bounds  $s_j^l\leqslant s_j\leqslant s_j^u$ on the DWI signals $s_j$, ($j=0,\ldots,N$). Then the bounds for $g_j=\log(s_j/s_0)$ are
\begin{equation}
    \label{eq:gj}
    g_j^l=\log(s_j^l/s_0^u);\quad g_j^u=\log(s_j^u/s_0^l),
    \quad (j=1,\ldots,N),
\end{equation}
because $g_j$ is monotone in regards to $s_j$.
We are thus trying to solve
\begin{equation}
    \label{eq:detmin-num}
    u=\argmin_{u'\in U=\cap U^j} R(u')
\end{equation}
where
\begin{equation*}
    U^j=\{u:g_j^l\leqslant A_ju\leqslant g_j^u\}
\end{equation*}

For the ease of notation, we write
\begin{align}
    \notag
    g&=\begin{pmatrix}g_1, & \ldots, & g_N\end{pmatrix}, \quad\text{and} \\
    Au & = \begin{pmatrix} A_1 u_1, & \ldots, & A_N u_N\end{pmatrix},
\end{align}
so that the Stejskal-Tanner equation is satisfied with
\[
    Au=g.
\]
The problem \eqref{eq:detmin-num} may be rewritten as 
\[
    \min_{u'}~ F_0(Au')+R(u'),
\]
for
\[
    \notag F_0(y) = \delta(g^l \leqslant y\leqslant g^r)
\]
with $\delta(g^l \leqslant y\leqslant g^u)$ denoting the indicator function of the convex set $\{y \colon g^l \leqslant y\leqslant g^r\}$. Solving the conjugate
\[
    F_0^*(y)=\begin{cases} \iprod{g^l}{y}, & y<0 \\ \iprod{g^u}{y}, & y\geq 0. \end{cases},
\]
and also writing
\[
    R(u)=R_0(K_0 u)
\]
for some $R_0$ and $K_0$, the problem can further be written in the saddle point form \eqref{eq:saddle} with
\[
    \begin{aligned}
        G(u) & = 0, \\
        K & = \begin{pmatrix} A \\ K_0 \end{pmatrix}, \\
        F^*(y,\psi) & = F_0^*(y) + R_0^*(\psi).
    \end{aligned}
\]
To apply algorithm \eqref{eq:cp}, we need to compute the resolvents of $G_0^*$ and $R_0^*$. For details regarding $R_0^*$ for $R=\TGV_{(\beta,\alpha)}^2$ and $R=\alpha \TV$ in the discretised setting, we refer to \cite{tuomov-dtireg,ipmsproc}; here it suffices to note that for $R=\alpha\TV$, we have $K_0=\alpha E$ and $R_0^*(\phi)$ is the indicator function of the dual ball $\{\phi \mid \sup_x \norm{\phi(x)}_F \le 1\}$. Thus the resolvent $\inv{(I+\tau \subdiff R_0^*)}$ becomes a projection to the dual ball. The case of $R=\TGV_{(\beta,\alpha)}^2$ is similar.
For $F_0^*$ we have
\begin{align}
    \notag
    (I+\tau \subdiff F_0^*)^{-1}(y) & = \argmin_{y'} F_0^*(y')+\frac{|y-y'|^2}{2\tau},
\end{align}
which resolves pointwise at each $\xi \in \Omega$ into the expression
\[
    [(I+\tau \subdiff F_0^*)^{-1}(y)](\xi)=S(y(\xi))
\]
for
\begin{equation*}
    \notag
    S(y(\xi))=\begin{cases} y(\xi)-g^u(\xi)\sigma, & y(\xi)\geq g^u(\xi)\sigma, \\ 0, & g^u(\xi)\sigma\leq y(\xi)\leq g^l(\xi)\sigma, \\ y(\xi)-g^l(\xi)\sigma, &  y(\xi)\leq g^l(\xi)\sigma. \end{cases}
\end{equation*}

Finally, we note that the saddle point system \eqref{eq:saddle} has to have a solution for the Chambolle--Pock algorithm to converge. In our setting, in particular, we need to find error bounds $g^l$ and $g^u$, for which there exists a solution $u$ to 
\begin{equation}
    \label{eq:bound-consistency-num}
    g^l \leqslant Au \leqslant g^u.
\end{equation}
If one uses at most six independent diffusion directions ($N=6$), as we will, then, for any $g$, there in fact exists a solution to $g=Au$. The condition \eqref{eq:bound-consistency-num} becomes $g^l \leqslant g^u$, immediately guaranteed through the monotonicity of \eqref{eq:gj}, and the trivial conditions $s_j^l \leqslant s_j^u$.

We are thus ready to apply the algorithm \eqref{eq:cp} to diffusion tensor imaging with deterministic error bounds. For the realisation of \eqref{eq:cp} for models \eqref{eq:dti-recons-lin1}, \eqref{eq:dti-recons-lin2}, and \eqref{eq:dti-recons-nl}, we refer to \cite{tuomov-dtireg,ipmsproc,escoproc,tuomov-nlpdhgm}.

\section{Experimental results}
\label{sec:experiments}

We now study the efficiency of the proposed reconstruction model in comparison to the different $L^2$-squared models, i.e., ones with Gaussian error assumption. This is based on a synthetic data, for which a ground-truth is available, as well as a real in vivo DTI data set. First we, however, have to describe in detail the procedure for obtaining upper and lower bounds for real data, when we do not know the true noise model, and are unable to perform multiple experiments as required by the theory of Section~\ref{sec:confidence}.

\subsection{Estimating lower and upper bounds from real data}
\label{sec:distr-estim}

As we have already discussed, in practice the noise in the measurement signals $\hat s_j$ is not Gaussian or Rician; in fact we do not know the true noise distribution and other corruptions. Therefore, we have to estimate the noise distribution from the image background. To do this, we require a known correspondence between the measurement, the noise, and the true value. As we have no better assumptions available, the standard one that we use is that of additive noise.
Continuing in the statistical setting of Section~\ref{sec:confidence}, we now describe the procedure, working on discrete images expressed as vectors $\hat f=s_j \in \R^n$ for some fixed $j \in \{0,1,\dots,N\}$. We use superscripts to denote the voxel indices, that is $\hat f=(f^1,\ldots,f^n)$.

In the $i$-th voxel, the measured value $\hat f^i$ is the sum of the true value $f^i$ and additive noise $\nu^i$:
\begin{equation*}
\hat f^i = f^i + \nu^i.
\end{equation*} 
All $\nu^i$ are assumed independent and identically distributed (i.i.d.), but their distribution is unknown. If we did know the true underlying cumulative distribution function $F$ of the noise, we could choose a confidence parameter $\theta \in (0,1)$ and use the cumulative distribution function to calculate $\nu_{\theta/2}, \nu_{1-\theta/2}$ such that\footnote{Recall that for a random variable $X$ with a cumulative distribution function $F$, the quantile function $\inv F$ returns a number $x_\theta=\inv F(\theta)$ such that $P(X \leqslant x_\theta) = \theta$.}
\begin{equation}\label{eq:prob_interval}
P(\nu_{\theta/2} \leqslant \nu^i \leqslant \nu_{1-\theta/2}) = 1-\theta.
\end{equation}

Let us instead proceed non-parametrically, and divide the whole image into two groups of voxels - the ones belonging to the background region and the rest. For simplicity, let the indices $i=1,\ldots,k$, ($k < n$), stand for the background voxels. In this region, we have $f^i = 0$ and $\hat f^i = \nu^i$. Therefore, the background region provides us with a number of independent samples from the unknown distribution of $\nu$.
The Dvoretzky--Kiefer--Wolfowitz inequality \cite{dvoretzky1956,massart1990,lehmann2008testing} states that
\[
    P\bigl(\sup_t \abs{F(t)-F_k(t)} > \epsilon\bigr) \le 2e^{-2k\epsilon^2},
\]
for the empirical cumulative distribution function
\[
    F_k(t) \defeq \frac{1}{k} \sum_{i=1}^k \chi_{(-\infty, \hat f^i]}(t).
\]
Therefore, computing $\nu_{\theta/2}$ and $\nu_{1-\theta/2}$ such that
\begin{equation}\notag
    P_k(\nu_{\theta/2} \leqslant \nu \leqslant \nu_{1-\theta/2}) = 1-\theta,
\end{equation}
we also have
\begin{equation}\notag
    P_k(f^i + \nu_{\theta/2} \leqslant \hat f^i \leqslant f^i + \nu_{1-\theta/2}) = 1-\theta.
\end{equation}
We may therefore use the values
\begin{equation}\label{eq:bounds_alpha}
    \hat f^{l,i} = \hat f^i - \nu_{1-\theta/2}, \quad
    \hat f^{u,i} = \hat f^i - \nu_{\theta/2},
\end{equation}
as lower and upper bounds for the true values $f^i$ outside the background region.

The Dvoretzky--Kiefer--Wolfowitz inequality implies that the interval estimates converge to the true intervals, determined by \eqref{eq:prob_interval}, as the number of background pixels $k$ increases with the image size $n$. This procedure, with large $k$, will therefore provide an estimate of a single-experiment ($m=1$) confidence interval for $f^i$. 
We note that this procedure will, however, not yield the convergence of the interval estimate $[\hat f^{l,i}, \hat f^{u,i}]$ to the true data; for that we would need multiple experiments, i.e., multiple sample images ($m>1$), not just agglomeration of the background voxels into a single noise distribution estimate. In practice, however, we can only afford a single experiment ($m=1$), and cannot go to the limit. 

\subsection{Verification of the approach with synthetic data}
To verify the effectiveness of the considered approach and to compare it to the other models, we use synthetic data. For the ground-truth tensor field $u_{g.t.}$ we take a helix region in a 3D box $100 \times 100 \times 30$, and choose the tensor in each point inside the helix in such a way that the principal eigenvector coincides with the helix direction (Figure \ref{fig:helix}). The helix region is described by the following equations:
\begin{eqnarray*}
&&x=(R+r\cos(\theta))\cos(\phi),\\
&&y=(R+r\cos(\theta))\sin(\phi),\\
&&z=r\sin(\theta)+\phi/\phi_{\max},\\
&&\phi \in [0,\phi_{\max}],\quad r \in [0,r_{\max}], \quad \theta \in [0,2\pi].
\end{eqnarray*}
The vector direction in every point coincides with helix direction:
\begin{equation*}
\vec{r}=\begin{pmatrix}-R\sin(\phi)\\R\cos(\phi)\\1/\phi_{\max}\end{pmatrix}.
\end{equation*}
We take the parameters $R=0.3, \phi_{\max}=4\pi, r_{\max}=0.07$ in this numerical experiment.

\begin{figure}
\centering
\includegraphics[width=0.8\textwidth]{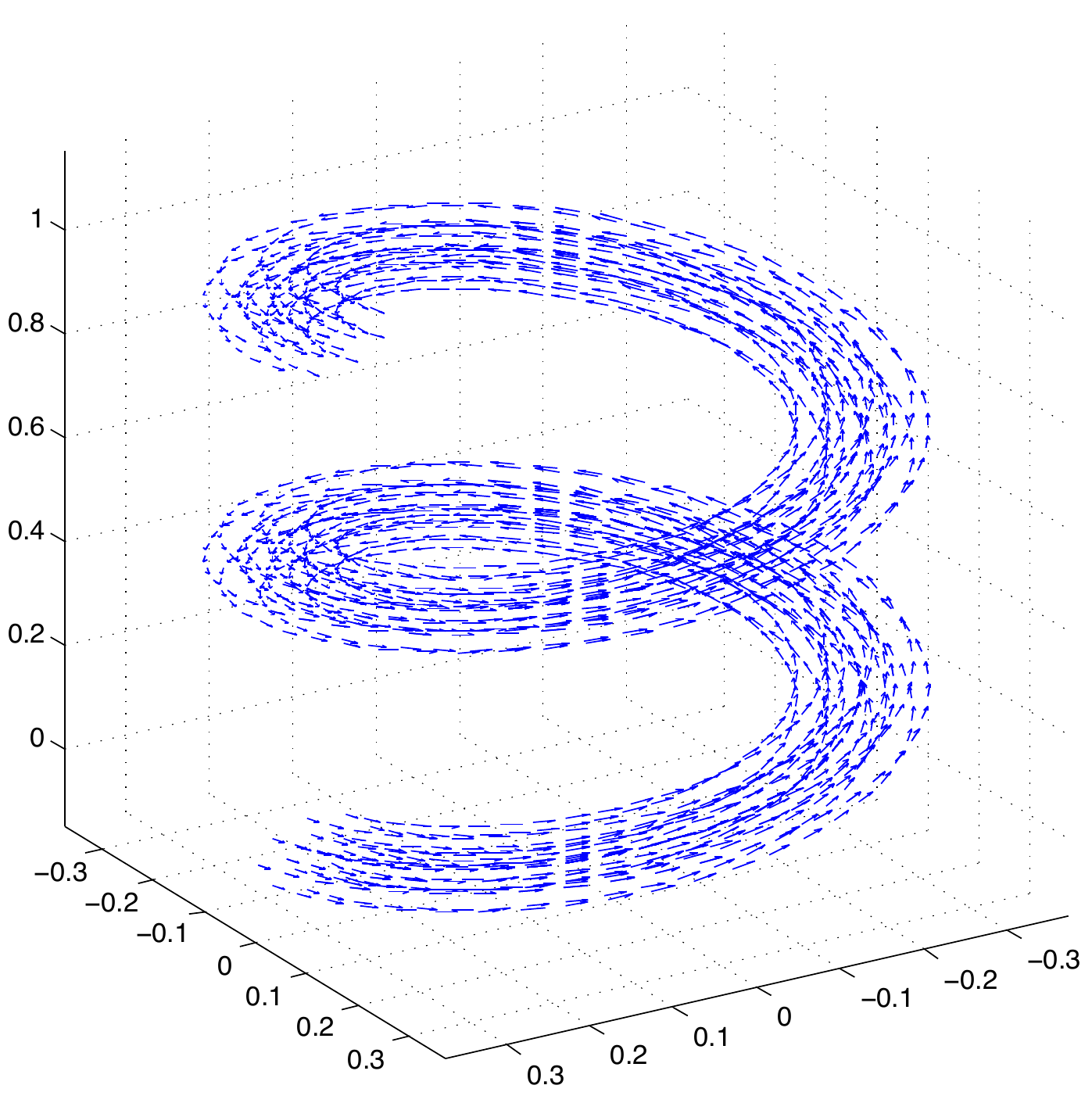}
\caption{Helix vector field for the principal eigenvector of the ground-truth tensor field}
\label{fig:helix}
\end{figure}
We apply the forward operators $T_j(u_{g.t.})$, ($j=0, \ldots, 6$), to obtain the data $s_j(x)$. We then add Rician noise to this data $\bar{s_j}=s_j+\delta$ with $\sigma=2$, which corresponds to $\PSNR\approx 27\mbox{dB}$.

We apply several models for solving the inverse problem of reconstructing $u$: the linear and non-linear $L^2$ approaches \eqref{eq:dti-recons-lin1} and \eqref{eq:dti-recons-nl}, and the constrained problem \eqref{eq:dti-recons-constr}. As the regulariser we use $R=\TGV^2_{(0.9\alpha,\alpha)}$, where the choice $\beta=0.9\alpha$ was made somewhat arbitrarily, however yielding good results for all the models. This is slightly lower than the range $[1, 1.5]\alpha$ discovered in comprehensive experiments for other imaging modalities \cite{tuomov-phaserec,tuomov-tgvlearn}.

For the linear and non-linear $L^2$ models \eqref{eq:dti-recons-lin1} and \eqref{eq:dti-recons-nl}, respectively, the regularisation parameter $\alpha$ is chosen either by a version of the discrepancy principle for inconsistent problems~\cite{TGSYag} or optimally with regard to the $\norm{\freevar}_{F,2}$ distance between the solution and the ground-truth. In case of the discrepancy principle, such an $\alpha$ was chosen that the following equality holds:
\begin{equation}
\Delta \rho(\alpha)=\sum_j||T_j(u)-\bar{s_j}||^2-\tau\sum_j||\bar{s_j}-s_j||^2=0
\end{equation}

We find $\alpha$ by solving this equation numerically using bisection method. We start by finding such $\alpha_1,\alpha_2$ that $\Delta \rho(\alpha_1)>0$ and $\Delta \rho(\alpha_2)<0$. We calculate $\Delta \rho(\alpha_3), \alpha_3=\frac{\alpha_1+\alpha_2}{2}$ and depending on its sign replace either $\alpha_1$ or $\alpha_2$ with $\alpha_3$. We repeat this procedure until the stopping criteria is reached.

As stopping criteria we use $|f(\alpha)|<\epsilon$. We use $\tau=1.05$, $\epsilon=0.01$  for linear and $\tau=1.2$, $\epsilon=0.0001$ for non-linear $L^2$ solution. A value of $\tau$ yielding a reasonable degree of smoothness has been chosen by trial and error, and is different for the non-linear model, reflecting a different non-linear objective in the discrepancy principle. For the constrained problem we calculate $\theta=90\%$, $95\%$, and $99\%$ confidence intervals to generate the upper and lower bounds. We however digress a little bit from the approach of Section~\ref{sec:confidence}. Minding that we do not know the true underlying distribution, which fails to be Rician as illustrated in Figure \ref{fig:noise-illustration}, we do not use it to calculate the confidence intervals, but use the estimation procedure described in Section~\ref{sec:distr-estim}. We stress that we only have a single sample of each signal $s_j$, so are unable to verify any asymptotic estimation properties.

The numerical results are in Table \ref{table:synthetic} and Figures \ref{fig:synthetic-colordir}--\ref{fig:synthetic-fa-peangle}, with the first of the figures showing the colour-coded principal eigenvector of the reconstruction, the second showing the fractional anisotropy and principal eigenvectors, and the last one the errors in the latter two, in a colour-coded manner. All plots are masked to represent only the non-zero region. The field of fractional anisotropy is defined for a field $u$ of 2-tensors on $\Omega \subset \R^m$ as
\[
    \FA_u(x) =\Bigl(\textstyle\sum_{i=1}^m (\lambda_i-\bar \lambda)^2\Bigr)^{1/2}
              \Bigl({\textstyle\sum_{i=1}^m \lambda_i^2}\Bigr)^{-1/2}
        \in [0, 1],
    \quad
    (x \in \Omega),
\]
with $\lambda_1,\ldots,\lambda_m$ denoting the eigenvalues of $u(x)$. It measures how far the ellipsoid prescribed by the eigenvalues and eigenvectors is from a sphere, with $\FA_u(x)=1$ corresponding a full sphere, and $\FA_u(x)=0$ corresponding to a degenerate object not having full dimension.

\begin{table}
\centering
\caption{Numerical results for the synthetic data. For the linear and non-linear $L^2$ the free parameter chosen by the parameter choice criterion is the regularisation parameter $\alpha$, and for the constrained problem it is the confidence interval.}
\label{table:synthetic}
\footnotesize
\begin{tabular}{l|l|l|l|l}
Method&Parameter choice & \parbox[t]{1.6cm}{Frobenius PSNR}& \parbox[t]{1.6cm}{Pr. e.val.\\ PSNR}& \parbox[t]{1.7cm}{Pr. e.vect.\\ angle PSNR}\\ \hline
Regression& & 33.90dB & 25.04dB & 47.86dB\\
Linear $L^2$&Discr. Principle& 32.93dB& 27.81dB&61.89dB\\
Linear $L^2$&Frob. Error-optimal& 34.51dB& 28.42dB&60.93dB\\
Non-linear $L^2$&Discr. Principle& 37.33dB&  27.81dB&61.89dB\\
Non-linear $L^2$&Frob. Error-optimal& 37.44dB& 28.03dB& 61.12dB\\
Constraints&90\%& 32.28dB& 28.86dB& 65.65dB\\
Constraints&95\%& 30.97dB& 28.14dB& 64.80dB\\
Constrains &99\%& 27.86dB& 24.51dB& 61.41dB\\
\end{tabular}
\end{table}

\begin{figure}[!htp]
\centering
\begin{subfigure}[t]{0.23\textwidth}
\includegraphics[width=1\textwidth,bb=20 45 120 220,clip]{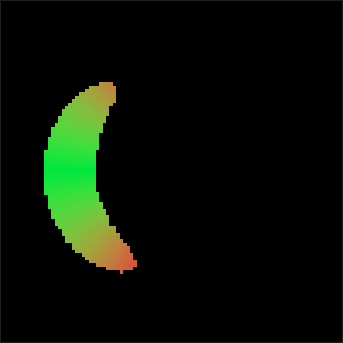}
\subcaption{Ground-truth}
\end{subfigure}
\begin{subfigure}[t]{0.23\textwidth}
\includegraphics[width=1\textwidth,bb=20 42 114 206,clip]{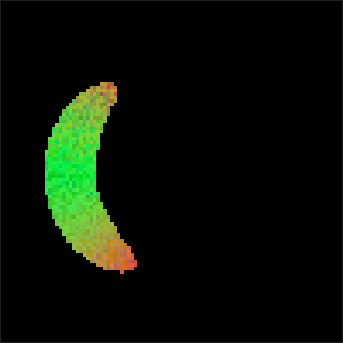}
\subcaption{Regression result}
\end{subfigure}
\begin{subfigure}[t]{0.23\textwidth}
\includegraphics[width=1\textwidth,bb=20 45 120 220,clip]{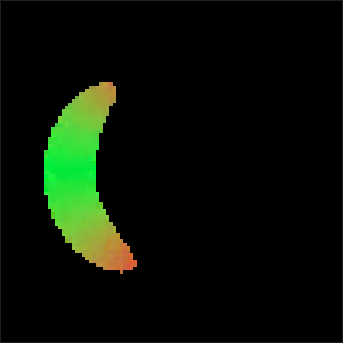}
\subcaption{Linear $L^2$, discrepancy principle}
\end{subfigure}
\begin{subfigure}[t]{0.23\textwidth}
\includegraphics[width=1\textwidth,bb=20 45 120 220,clip]{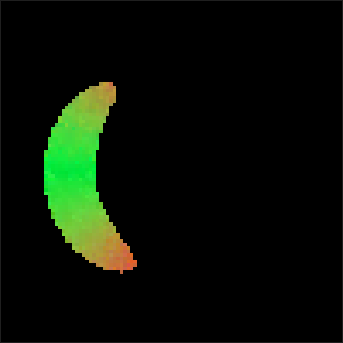}
\subcaption{Linear $L^2$, error-optimal}
\end{subfigure}
\\
\begin{subfigure}[t]{0.23\textwidth}
\includegraphics[width=1\textwidth,bb=20 45 120 220,clip]{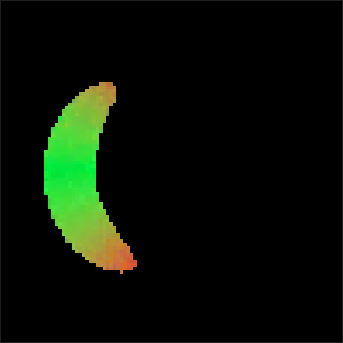}
\subcaption{Non-linear $L^2$, discrepancy principle}
\end{subfigure}
\begin{subfigure}[t]{0.23\textwidth}
\includegraphics[width=1\textwidth,bb=30 59 156 280,clip]{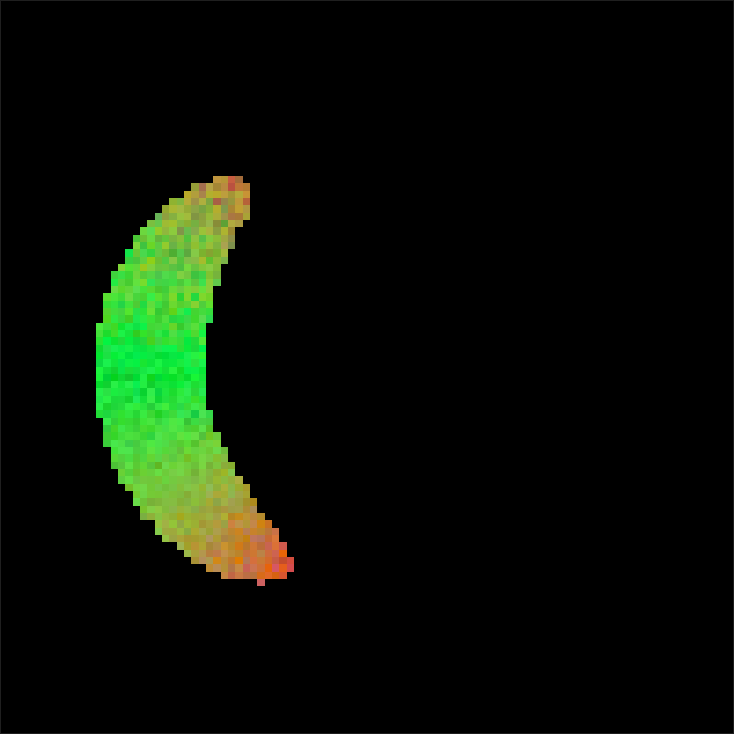}
\subcaption{Non-linear $L^2$, error-optimal}
\end{subfigure}
\begin{subfigure}[t]{0.23\textwidth}
\includegraphics[width=1\textwidth,bb=30 59 156 280,clip]{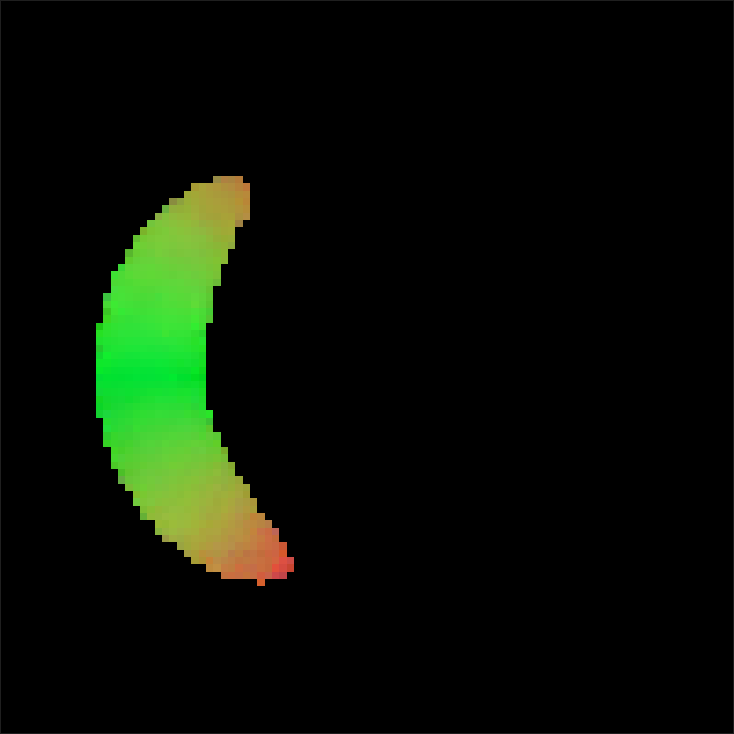}
\subcaption{Constrained, 95\% confidence intervals}
\end{subfigure}
\begin{subfigure}[t]{0.23\textwidth}
\includegraphics[width=1\textwidth,bb=20 45 120 220,clip]{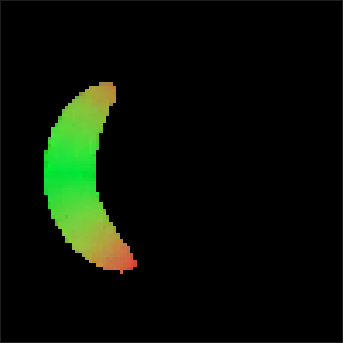}
\subcaption{Constrained, 90\% confidence intervals}
\end{subfigure}
\\
\begin{minipage}[c]{0.8\textwidth}
\caption{Synthetic test data results. (a) ground-truth plot. (b) regression result plot. (c)--(h) Plot of a slice of the solution for $L^2$, non-linear $L^2$ and constrained models. 
The legend on the right indicates the colour-coding of directions of the principal eigenvector plotted.
}
\label{fig:synthetic-colordir}
\end{minipage}%
\begin{minipage}[t]{0.2\textwidth}
\flushright%
\setlength{\w}{\textwidth}
\input{figures/legend-dirplot3d.tikz}
\null
\vfill
\end{minipage}
\end{figure}

\begin{figure}[!htp]
\centering%
\begin{subfigure}[t]{0.23\textwidth}
\includegraphics[width=1\textwidth,bb=40 70 180 330,clip]{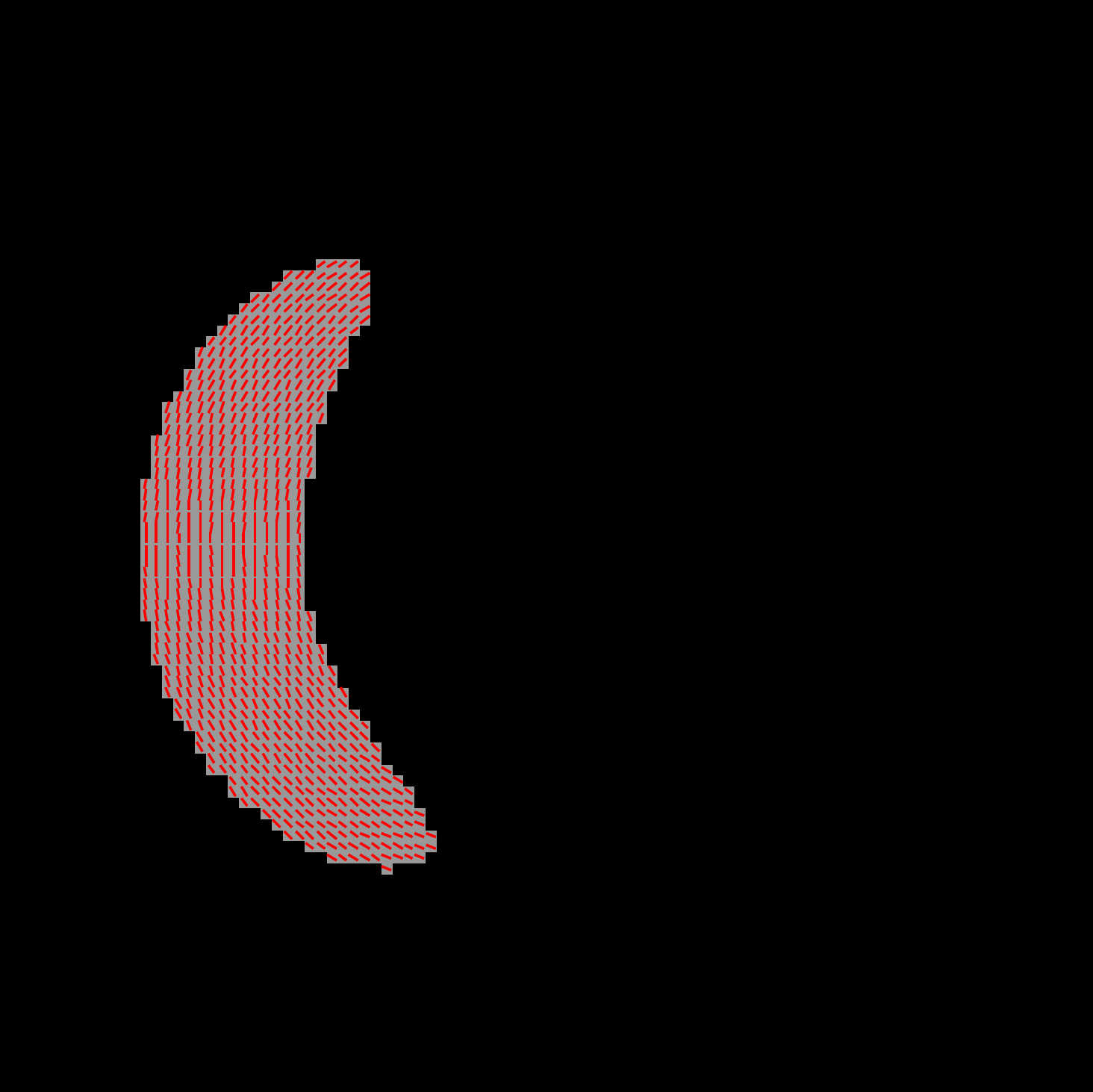}
\subcaption{Ground-truth}
\end{subfigure}
\begin{subfigure}[t]{0.23\textwidth}
\includegraphics[width=1\textwidth,bb=40 70 180 330,clip]{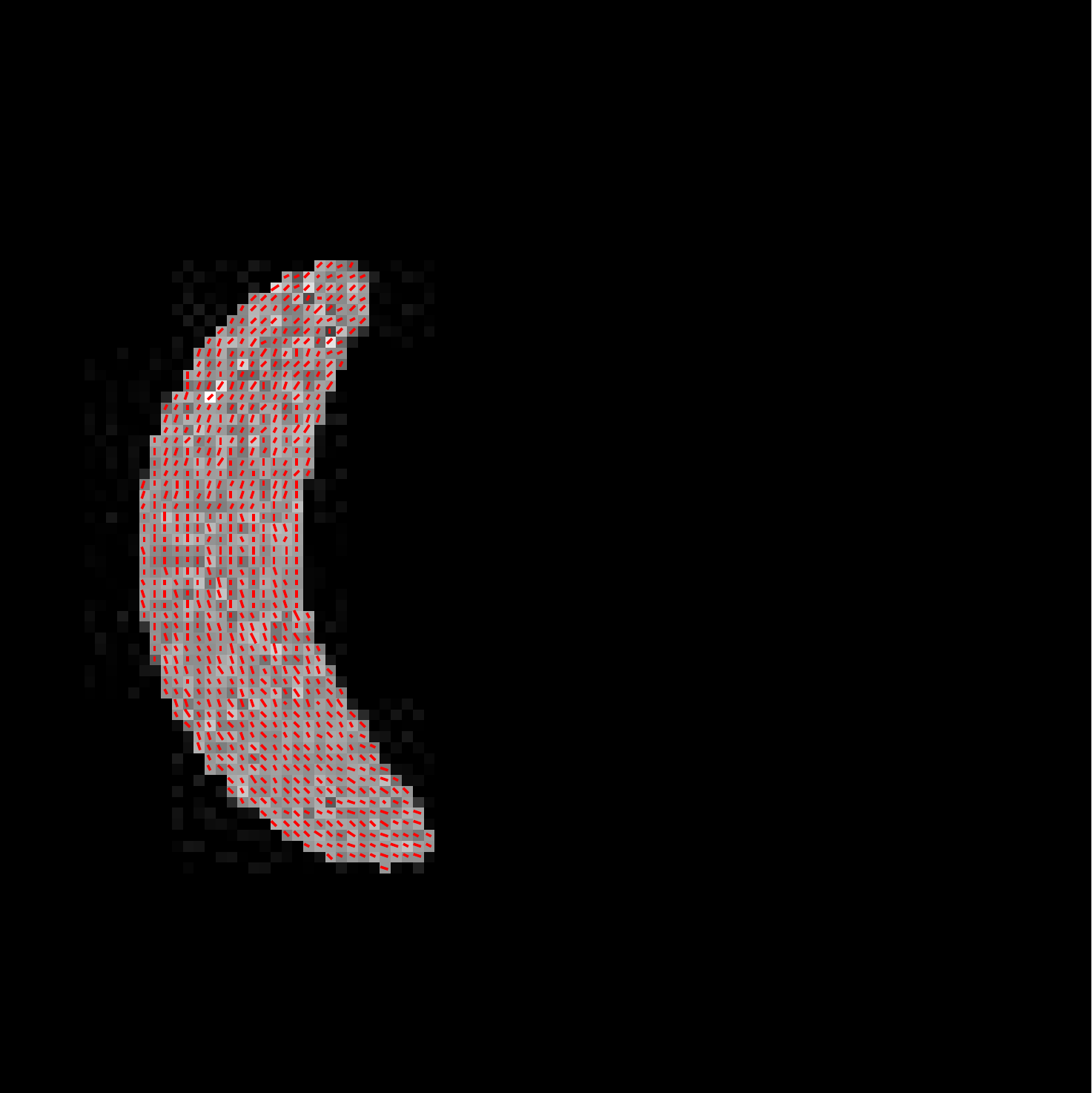}
\subcaption{Regression result}
\end{subfigure}
\begin{subfigure}[t]{0.23\textwidth}
\includegraphics[width=1\textwidth,bb=40 70 180 330,clip]{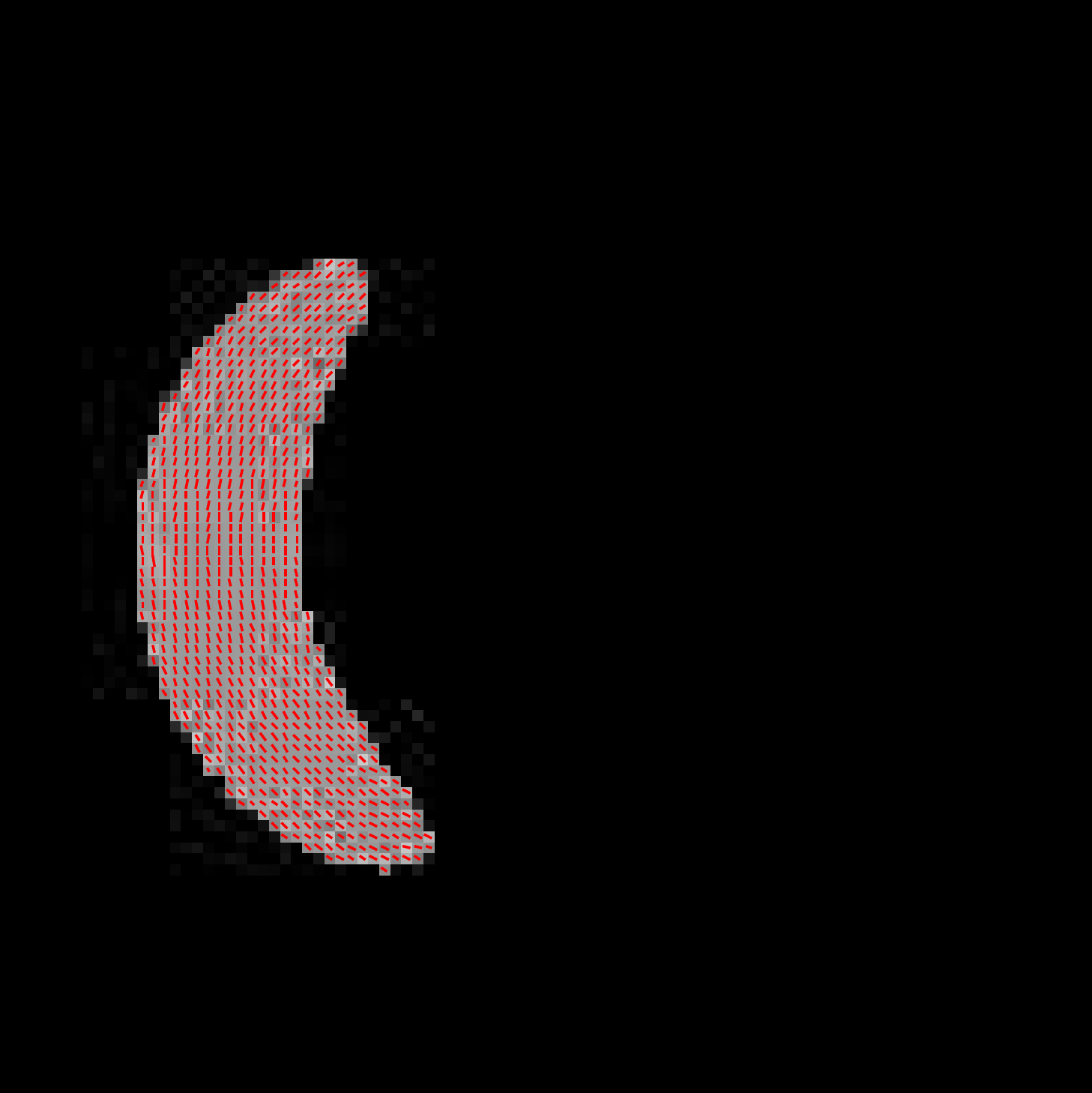}
\subcaption{Linear $L^2$, discrepancy principle}
\end{subfigure}
\begin{subfigure}[t]{0.23\textwidth}
\includegraphics[width=1\textwidth,bb=40 70 180 330,clip]{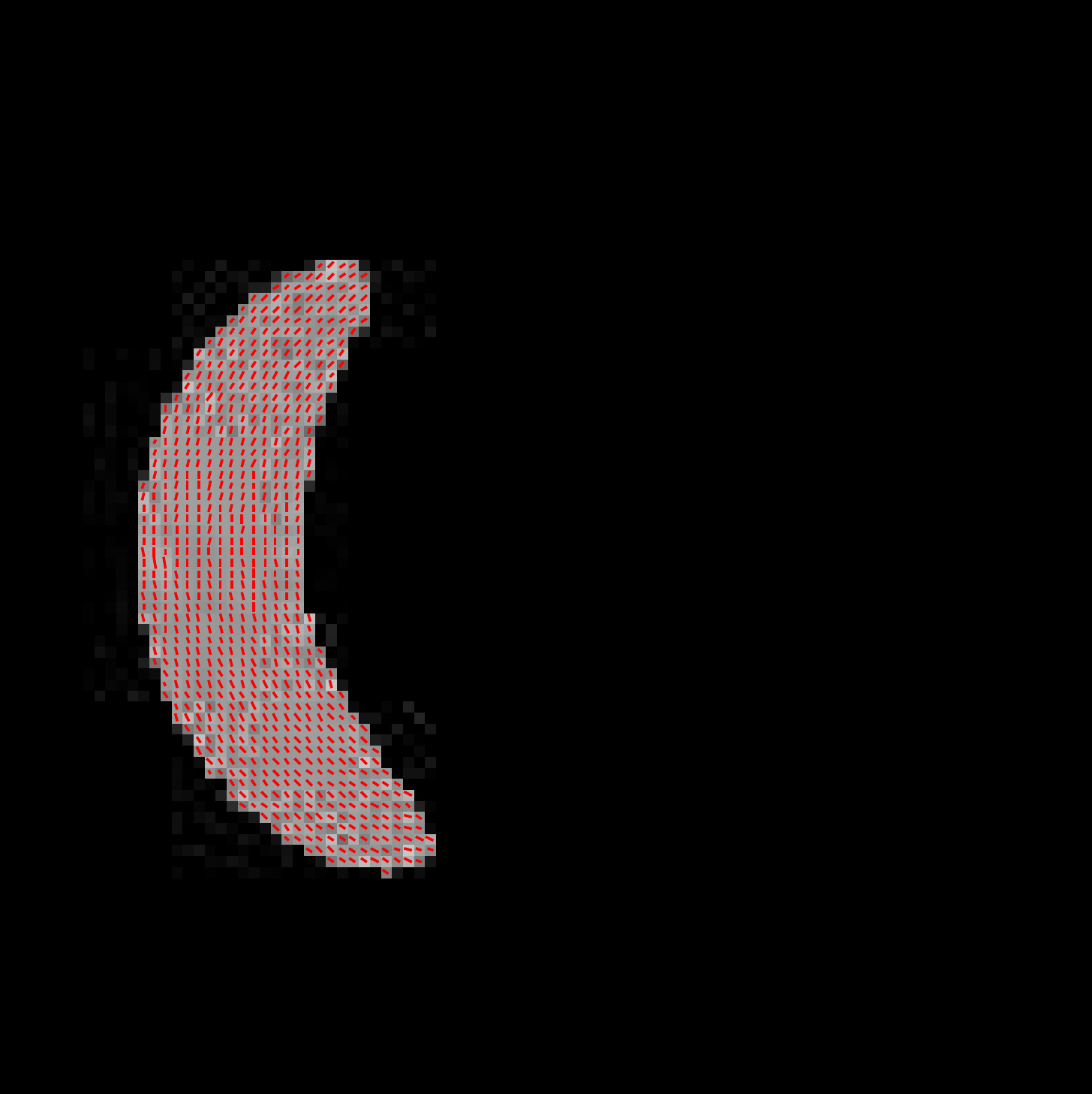}
\subcaption{Linear $L^2$, error-optimal}
\end{subfigure}
\begin{subfigure}[t]{0.23\textwidth}
\includegraphics[width=1\textwidth,bb=40 70 180 330,clip]{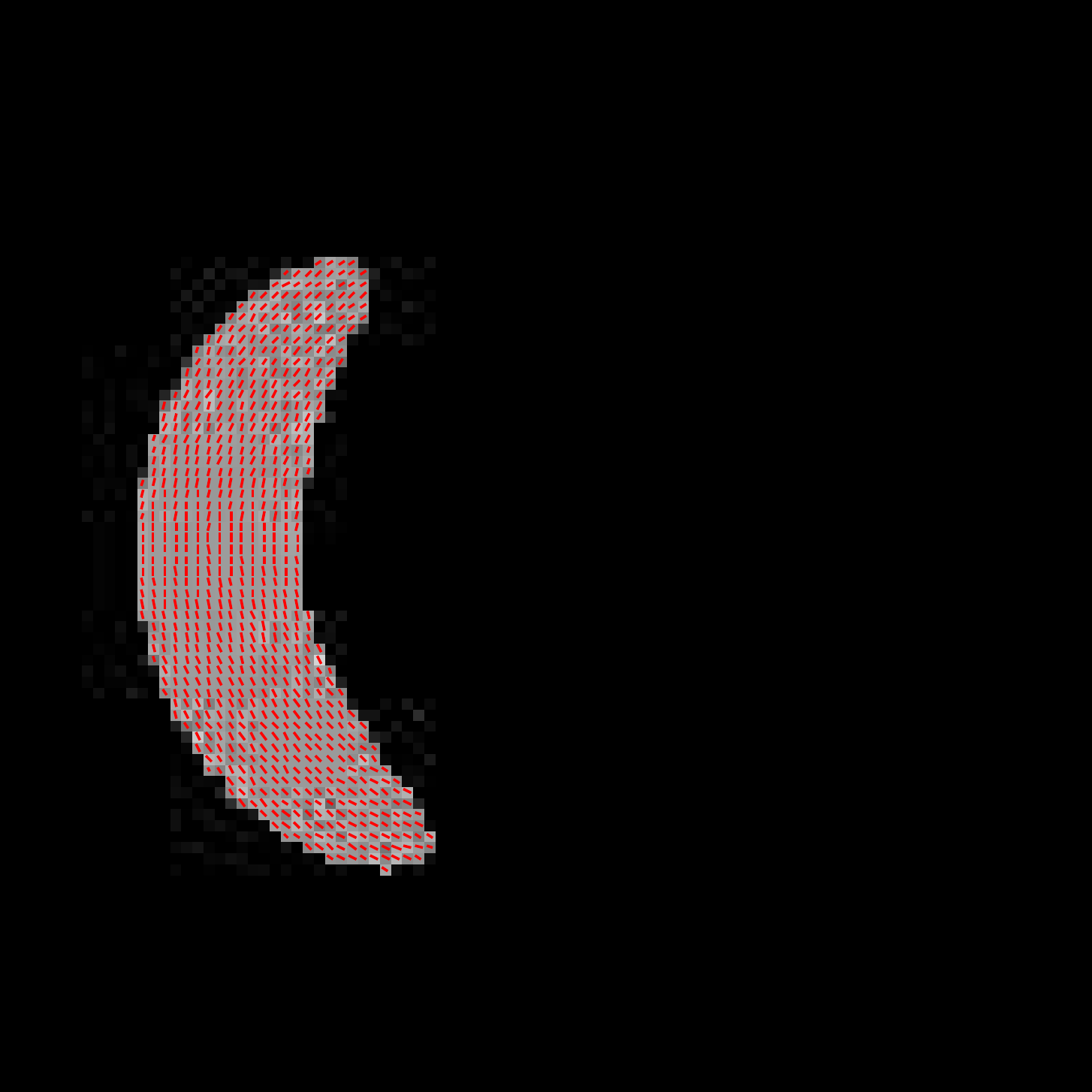}
\subcaption{Non-linear $L^2$, discrepancy principle}
\end{subfigure}
\begin{subfigure}[t]{0.23\textwidth}
\includegraphics[width=1\textwidth,bb=40 70 180 330,clip]{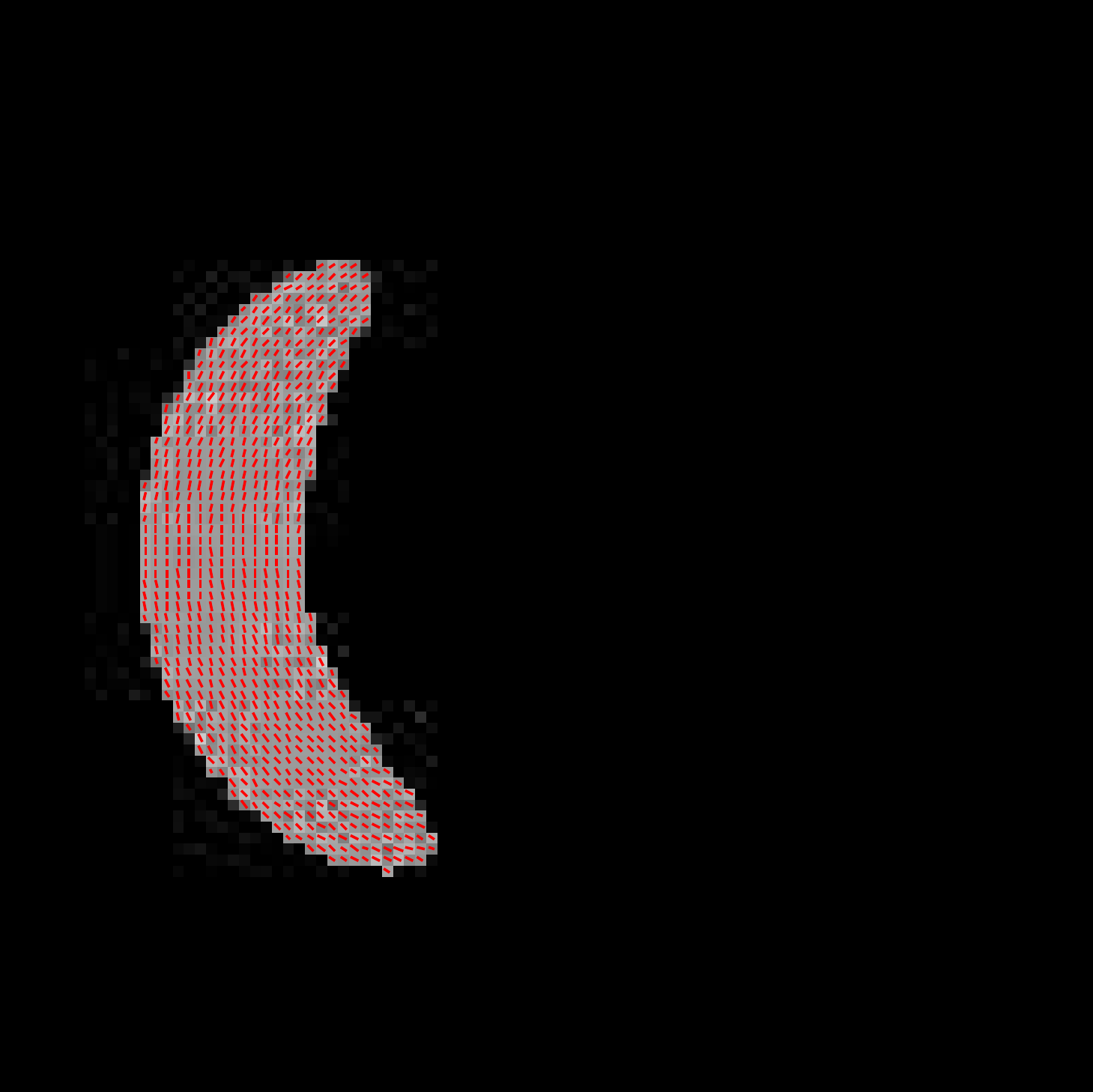}
\subcaption{Non-linear $L^2$, error-optimal}
\end{subfigure}
\begin{subfigure}[t]{0.23\textwidth}
\includegraphics[width=1\textwidth,bb=40 70 180 330,clip]{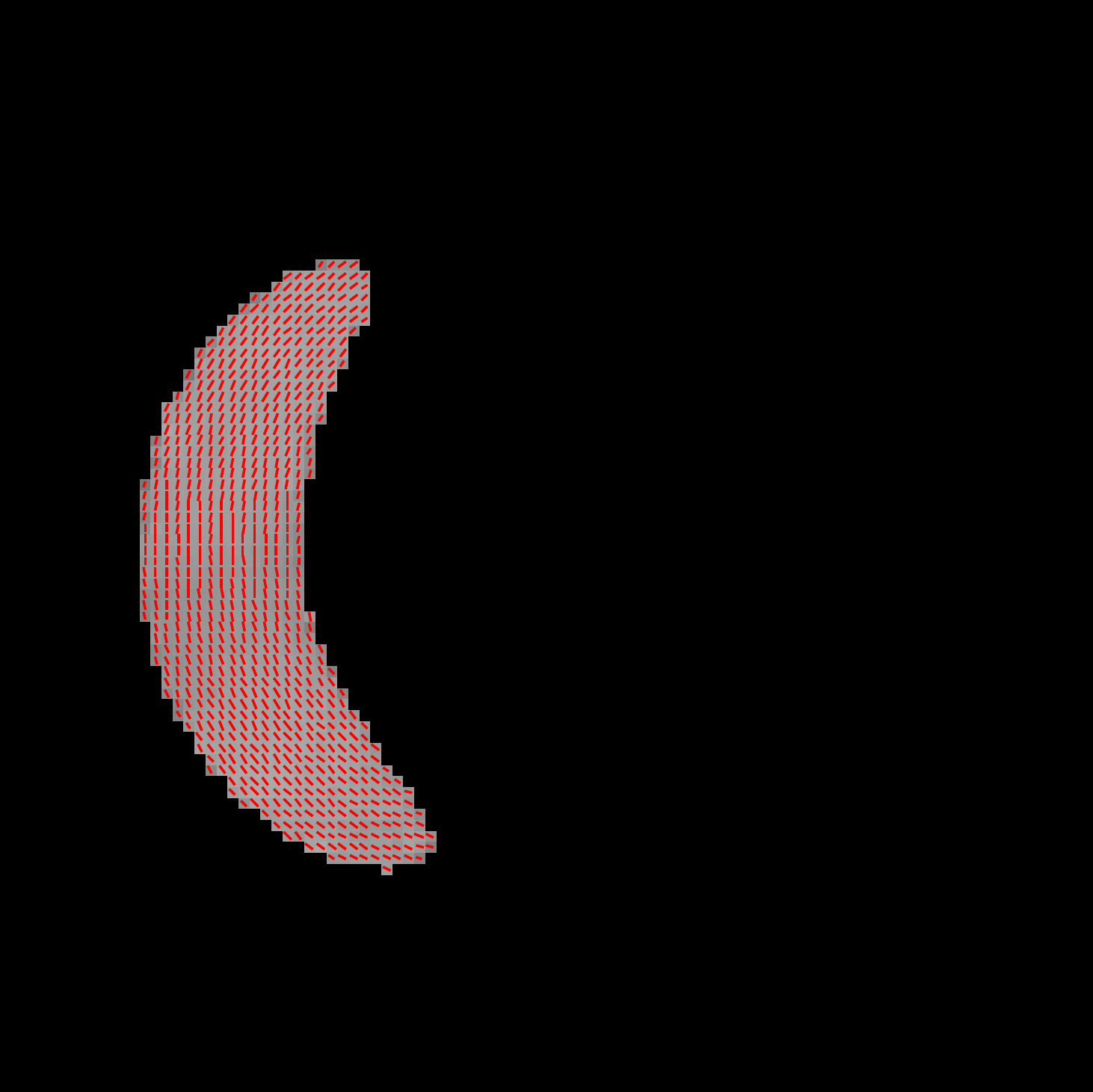}
\subcaption{Constrained, 95\% confidence intervals}
\end{subfigure}
\begin{subfigure}[t]{0.23\textwidth}
\includegraphics[width=1\textwidth,bb=40 70 180 330,clip]{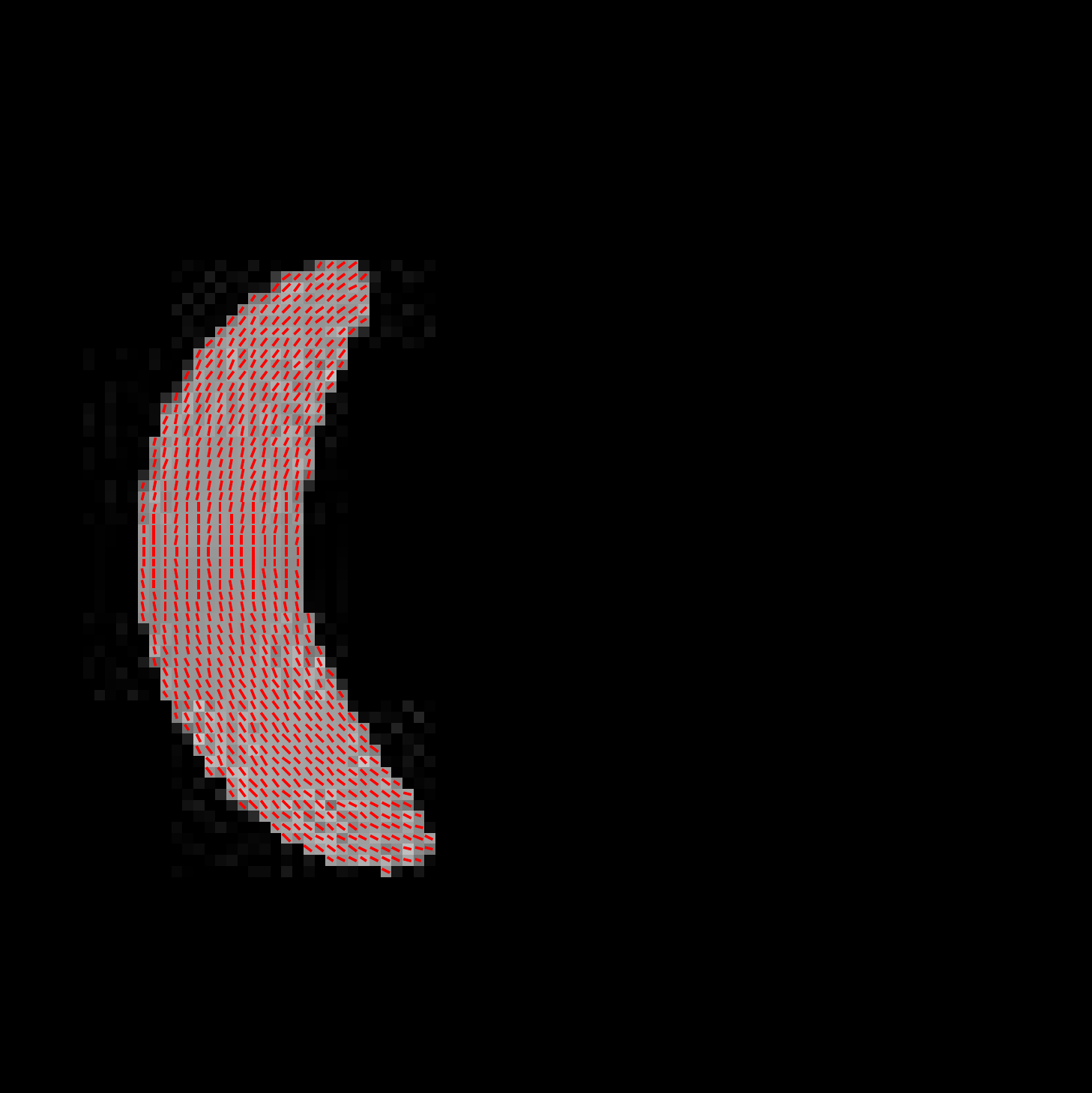}
\subcaption{Constrained, 90\% confidence intervals}
\end{subfigure}
\begin{minipage}[c]{0.8\textwidth}
\caption{Fractional anisotropy in greyscale superimposed by principal eigenvector.
Legend on left indicates the greyscale intensities of the fractional anisotropy.}
\label{fig:synthetic-fa-pev}
\end{minipage}%
\begin{minipage}[t]{0.2\textwidth}
\flushright%
\setlength{\w}{\textwidth}
\tikzexternaldisable
\begin{tikzpicture}
        \scriptsize

        \shade [shading=axis, left color=black, right color=white]
                (-0.3\w,0.05) rectangle (0.3\w,0.1\w);

        \draw[->] (-0.3\w,0)
                  node[below] {0} --
                  node[below] {$\FA$} (0.3\w,0)
                  node[below] {1};

        \def\mp#1#2{\begin{minipage}{#1}
                #2
                \legendstyle
                Greyscale-coding of the fractional anisotropy.
                Principal eigenvector is drawn in red.
        \end{minipage}}

\end{tikzpicture}
\tikzexternalenable
\null
\vfill
\end{minipage}
\end{figure}
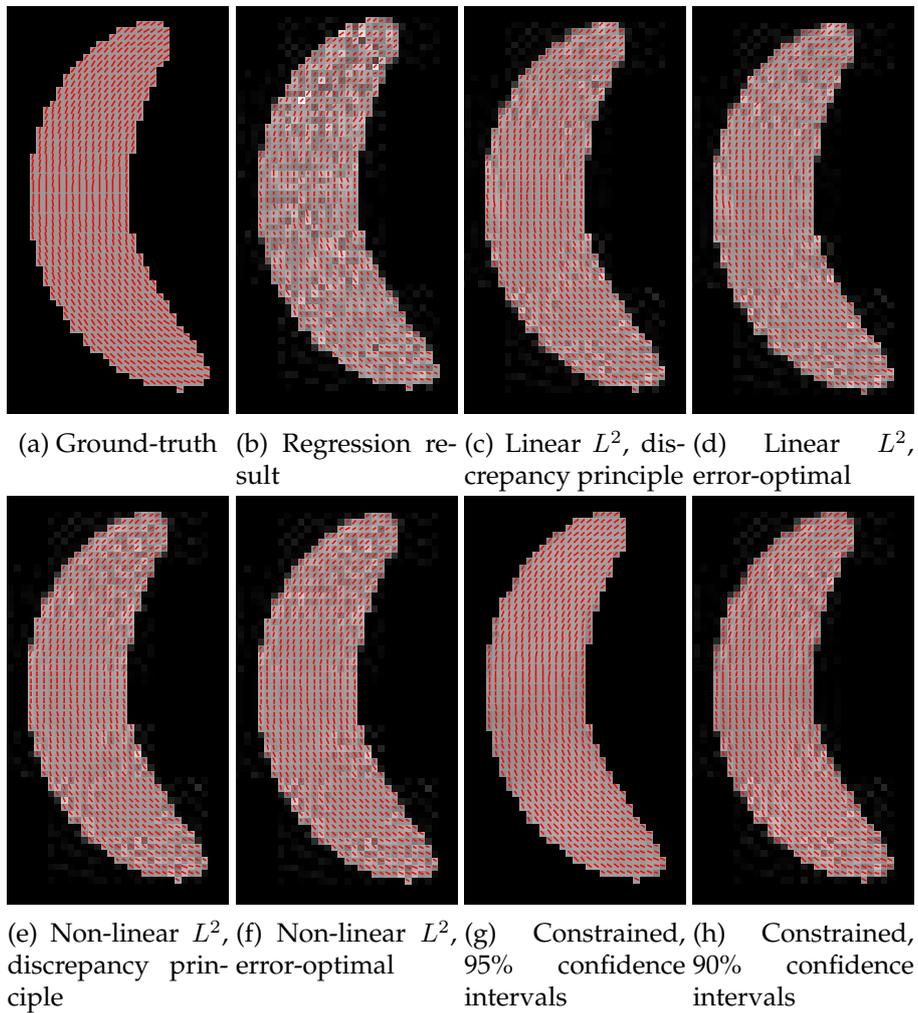

\begin{figure}[!htp]
\centering%
\begin{subfigure}[t]{0.23\textwidth}
\includegraphics[width=1\textwidth,bb=30 55 155 280,clip]{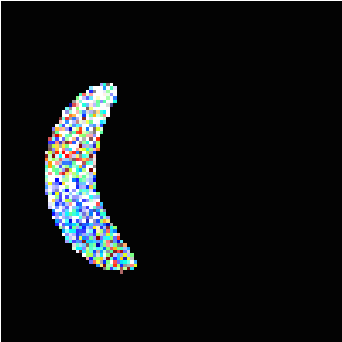}
\subcaption{Regression result}
\end{subfigure}
\begin{subfigure}[t]{0.23\textwidth}
\includegraphics[width=1\textwidth,bb=20 40 115 210,clip]{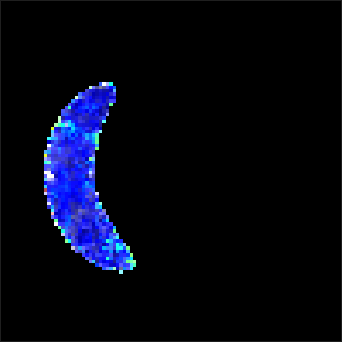}
\subcaption{Linear $L^2$, discrepancy principle}
\end{subfigure}
\begin{subfigure}[t]{0.23\textwidth}
\includegraphics[width=1\textwidth,bb=20 40 115 210,clip]{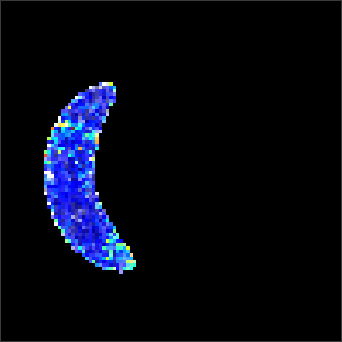}
\subcaption{Linear $L^2$, error-optimal}
\end{subfigure}
\begin{subfigure}[t]{0.23\textwidth}
\includegraphics[width=1\textwidth,bb=20 40 115 210,clip]{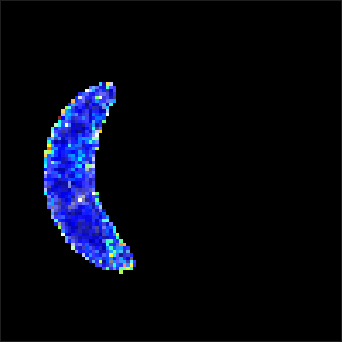}
\subcaption{Non-linear $L^2$, discrepancy principle}
\end{subfigure}
\\
\begin{subfigure}[t]{0.23\textwidth}
\includegraphics[width=1\textwidth,bb=20 40 115 210,clip]{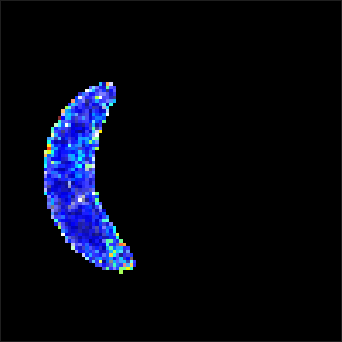}
\subcaption{Non-linear $L^2$, error-optimal}
\end{subfigure}
\begin{subfigure}[t]{0.23\textwidth}
\includegraphics[width=1\textwidth,bb=20 40 115 210,clip]{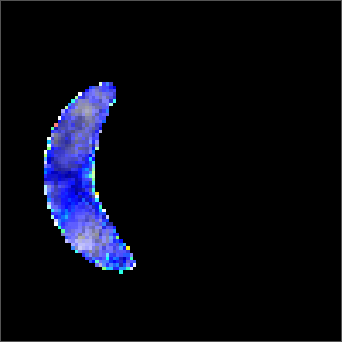}
\subcaption{Constrained, 95\% confidence intervals}
\end{subfigure}
\begin{subfigure}[t]{0.23\textwidth}
\includegraphics[width=1\textwidth,bb=20 40 115 210,clip]{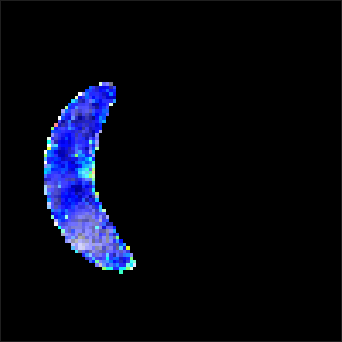}
\subcaption{Constrained, 90\% confidence intervals}
\end{subfigure}
\begin{minipage}[c]{0.8\textwidth}
\caption{Colour-coded errors of fractional anisotropy and principal eigenvector for the computations on the synthetic test data. Legend on the right indicates the
colour-coding of errors between
$u$ and $g_0$ as functions of the principal eigenvector angle error
\mbox{$\theta=\inv \cos(\iprod{\hat v_u}{\hat v_{g_0}})$} in terms of the hue, and the fractional anisotropy error
\mbox{$e_\FA=\abs{\FA_u -\FA_{g_\mathrm{0}}}$} in terms of whiteness.
}
\label{fig:synthetic-fa-peangle}
\end{minipage}%
\begin{minipage}[t]{0.2\textwidth}
\flushright%
\setlength{\w}{0.8\textwidth}
\input{figures/legend-errplot.tikz}
\null
\vfill
\end{minipage}
\end{figure}

As we can see, the non-linear approach~\eqref{eq:dti-recons-nl} performs overall the best by a wide margin, in terms of the pointwise Frobenius error, i.e., error in $\norm{\cdot}_{F,2}$. This is expressed as a PSNR in Table \ref{table:synthetic}. What is, however, interesting, is that the constraint-based approach~\eqref{eq:dti-recons-constr} has a much better reconstruction of the principal eigenvector angle, and a comparable reconstruction of its magnitude.
Indeed, the 95\% confidence interval in Figure \ref{fig:synthetic-colordir}(g) and Figure \ref{fig:synthetic-fa-pev}(g) suggests a nearly perfect reconstruction in terms of smoothness. But, the Frobenius PSNR in Table \ref{table:synthetic} for this approach is worse than the simple unregularised inversion by regression. The problem is revealed by Figure \ref{fig:synthetic-fa-peangle}(f): the large white cloudy areas indicate huge fractional anisotropy errors, while at the same time, the principal eigenvector angle errors expressed in colour are much lower than for other approaches.  Good reconstruction of the principal eigenvector is important for the process of tractography, i.e., the reconstruction of neural pathways in a brain.  One explanation for our good results is that the regulariser completely governs the solution in areas where the error bounds are inactive due to generally low errors. This results in very smooth reconstructions, which is in the present case desirable as our synthetic tensor field is also smooth within the helix.


\subsection{Results with in vivo brain imaging data}
\label{sec:invivo}

We now wish to study the proposed regularisation model on a real in-vivo diffusion tensor image. Our data is that of a human brain, with the measurements of a volunteer performed on a clinical 3T system (Siemens Magnetom TIM Trio, Erlangen, Germany), with a 32 channel head coil. A 2D diffusion weighted single shot EPI sequence with diffusion sensitising gradients applied in 12 independent directions ($b = 1000\mbox{s}/\mbox{mm}^2$). An additional reference scan without diffusion was used with the parameters: $\mbox{TR}=7900\mbox{ms}$, $\mbox{TE}=94\mbox{ms}$, flip angle $90^\circ$. Each slice of the 3D data set has plane resolution $1.95\mbox{mm} \times 1.95\mbox{mm}$, with a total of $128 \times 128$ pixels. The total number of slices is 60 with a slice thickness of 2mm.
The data set consists of 4 repeated measurements. The GRAPPA acceleration factor is 2.
Prior to the reconstruction of the diffusion tensor, eddy current correction was performed  with FSL \cite{smith2004advances}.
Written informed consent was obtained from the volunteer before the examination.

For error bounds calculation according to the procedure of Section~\ref{sec:distr-estim}, to avoid systematic bias near the brain, we only use about 0.6\% of the total volume near the borders, or roughly $k \approx 6000$ voxels.

To estimate errors for the all the considered reconstruction models, for each gradient direction $b_i$ we use only one out of the four duplicate measurements. We then calculate the errors using a somewhat less than ideal pseudo-ground-truth, which is the linear regression reconstruction from all the available measurements.

The results are in Table \ref{table:invivo} and Figures \ref{fig:invivo-colordir}--\ref{fig:invivo-fa-peangle}, again with the first of the figures showing the colour-coded principal eigenvector of the reconstruction, the second showing the fractional anisotropy and principal eigenvectors, and the last one the errors in the latter two, in a colour-coded manner. Again, all plots are masked to represent only the non-zero region. In the figures, we concentrate on error bounds based on 95\% confidence intervals, as the results for the 90\% and 99\% cases do not differ significantly according to Table \ref{table:invivo}.

\begin{table}[t]
\caption{Numerical results for the in-vivo brain data data. For the $L^2$ and non-linear $L^2$  reconstruction models the free parameter chosen by the parameter choice criterion is the regularisation parameter $\alpha$, and for the constrained problem it is the confidence interval.}
\label{table:invivo}
\centering
\footnotesize
\begin{tabular}{l|l|l|l|l}
Method&Parameter choice & \parbox[t]{1.6cm}{Frobenius PSNR}& \parbox[t]{1.6cm}{Pr. e.val.\\ PSNR}& \parbox[t]{1.7cm}{Pr. e.vect.\\ angle PSNR}\\ \hline
Regression& & 32.35dB& 33.67dB& 28.56dB \\
Linear $L^2$&Discr. Principle& 34.80dB& 36.35dB& 24.81dB\\
Linear $L^2$&Frob. Error-optimal& 34.81dB& 36.32dB& 24.97dB\\
Non-linear $L^2$&Discr. Principle& 33.53dB&  35.87dB& 27.12dB\\
Non-linear $L^2$&Frob. Error-optimal& 33.57dB&  36.03dB  & 27.58dB\\
Constraints&90\%&33.71dB &34.93dB & 27.00dB \\
Constraints&95\%& 33.70dB& 34.97dB& 26.91dB\\
Constraints&99\%& 33.67dB & 34.89dB & 26.88dB \\
\end{tabular}
\end{table}

\begin{figure}[!htp]
\begin{subfigure}[t]{0.32\textwidth}
\includegraphics[width=1\textwidth]{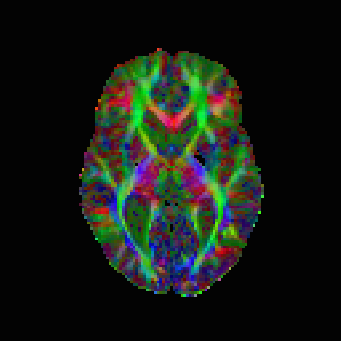}
\subcaption{Pseudo-ground-truth}
\end{subfigure}
\begin{subfigure}[t]{0.32\textwidth}
\includegraphics[width=1\textwidth]{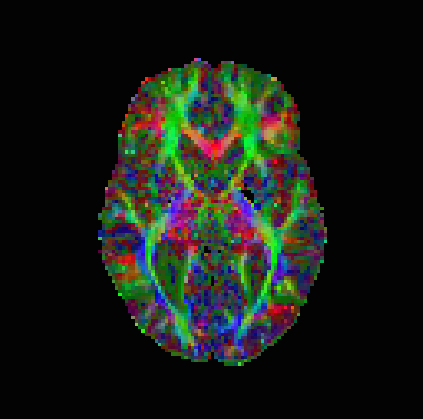}
\subcaption{Regression result}
\end{subfigure}
\begin{subfigure}[t]{0.32\textwidth}
\includegraphics[width=1\textwidth]{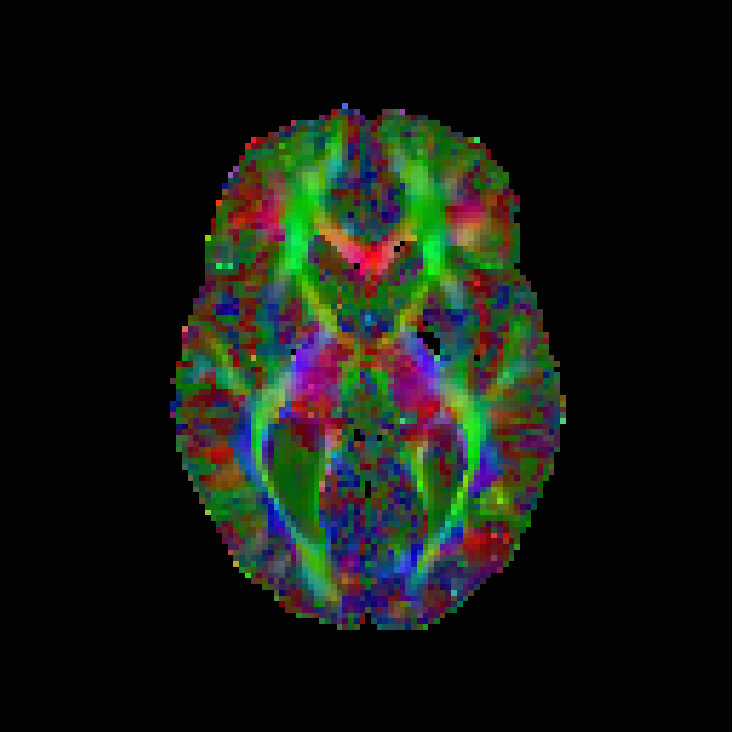}
\subcaption{Linear $L^2$, discrepancy principle}
\end{subfigure}
\begin{subfigure}[t]{0.32\textwidth}
\includegraphics[width=1\textwidth]{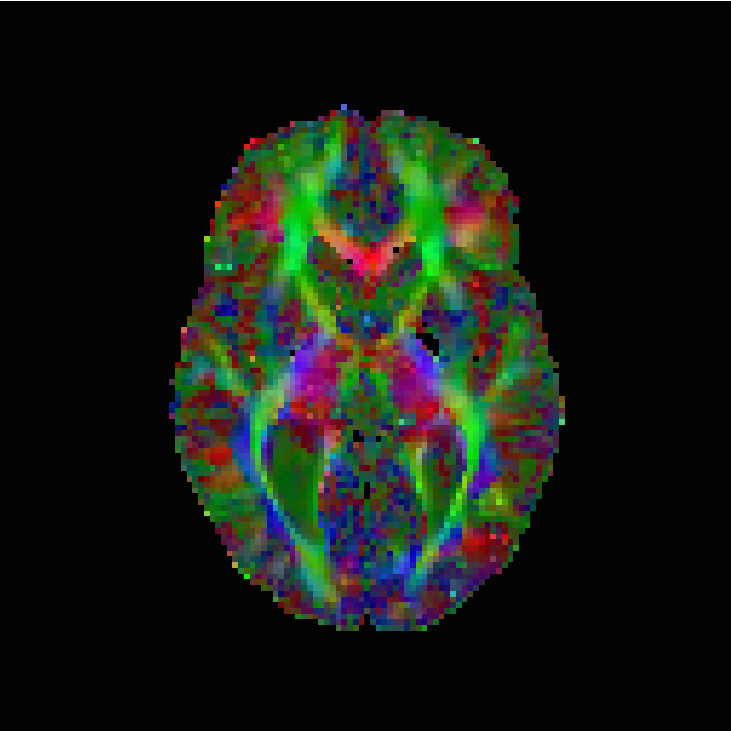}
\subcaption{Linear $L^2$, error-optimal}
\end{subfigure}
\begin{subfigure}[t]{0.32\textwidth}
\includegraphics[width=1\textwidth]{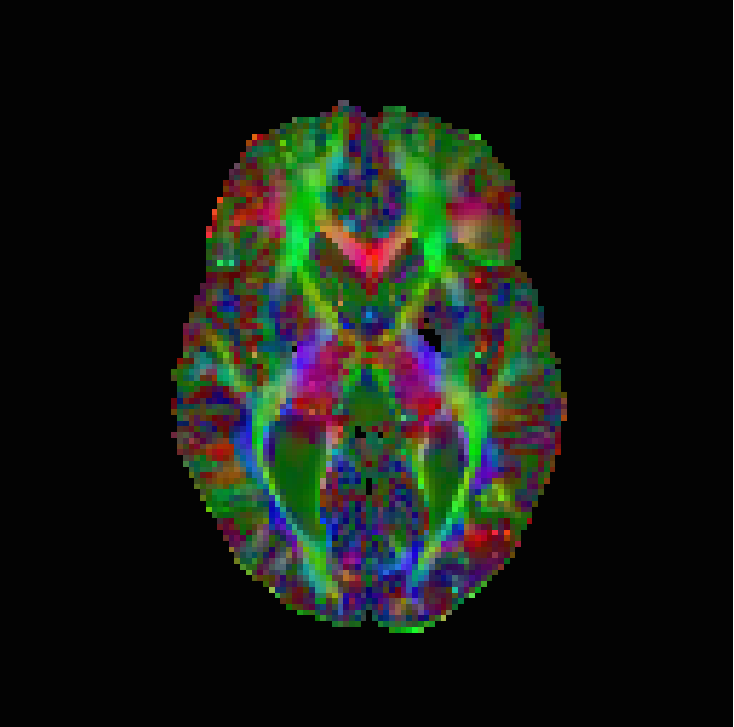}
\subcaption{Non-linear $L^2$, discrepancy principle}
\end{subfigure}
\begin{subfigure}[t]{0.32\textwidth}
\includegraphics[width=1\textwidth]{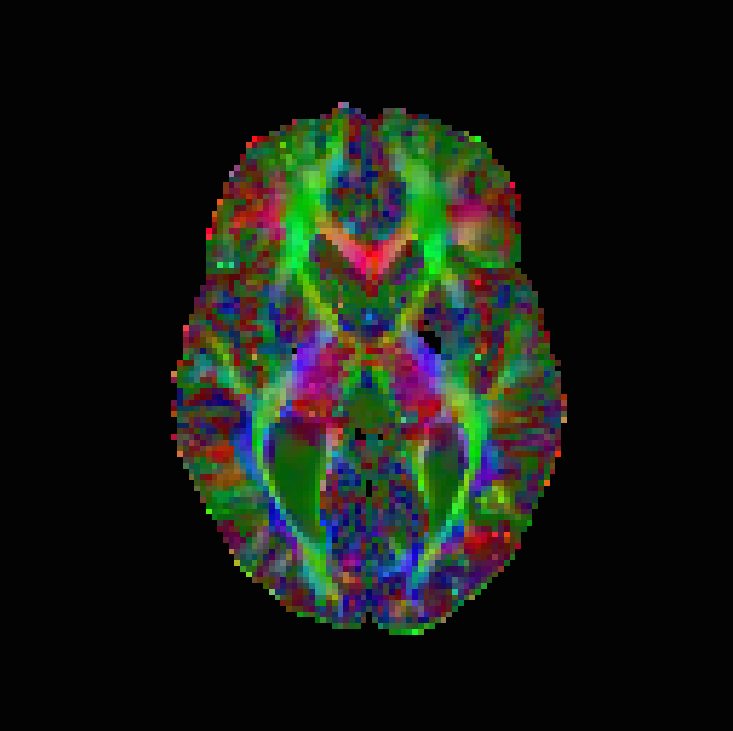}
\subcaption{Non-linear $L^2$, error-optimal}
\end{subfigure}
\begin{subfigure}[t]{0.32\textwidth}
\includegraphics[width=1\textwidth]{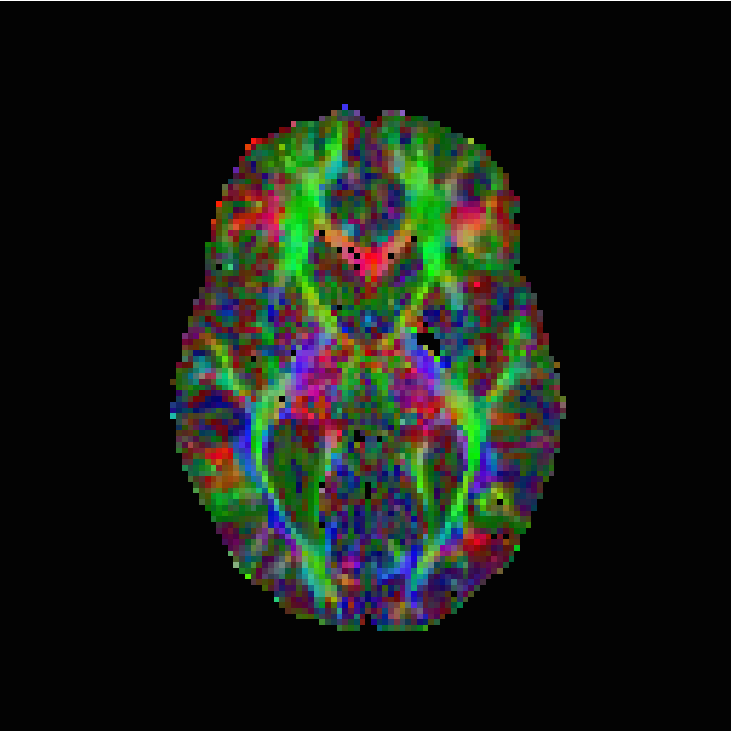}
\subcaption{Constrained, 95\% confidence intervals}
\end{subfigure}
\begin{subfigure}[b]{0.32\textwidth}
\hspace{\textwidth}
\end{subfigure}
\begin{subfigure}[b]{0.32\textwidth}
\setlength{\w}{0.7\textwidth}
\centering
\input{figures/legend-dirplot3d.tikz}
\vspace{1cm}
\end{subfigure}
\caption{Reconstruction results on the in vivo brain data. (a) Pseudo-ground-truth plot. (b) regression result. (c)--(g) Plot of a slice of the solution for $L^2$, non-linear $L^2$ and constrained problem approach. 
The legend on the bottom-right indicates the colour-coding of directions of the principal eigenvector plotted.
}
\label{fig:invivo-colordir}
\end{figure}

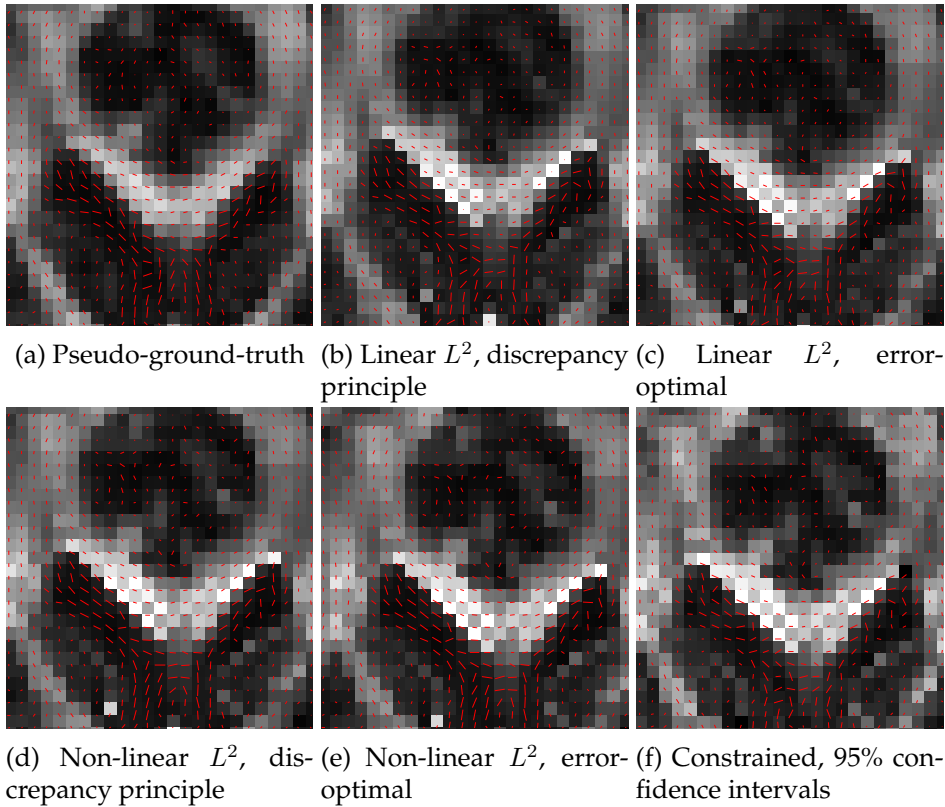
\begin{figure}[!htp]
\centering%
\begin{subfigure}[t]{0.32\textwidth}
\includegraphics[width=1\textwidth]{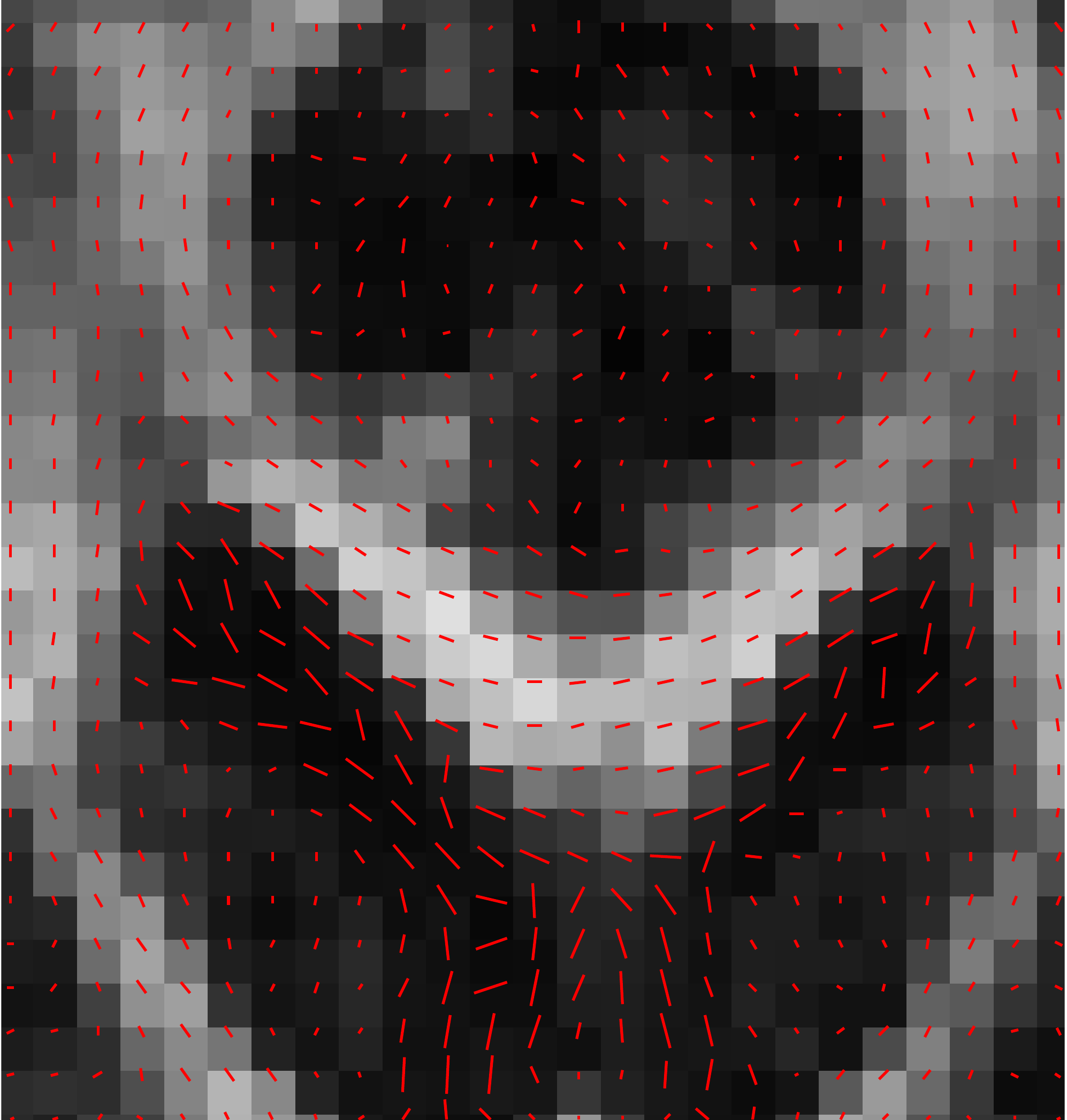}
\subcaption{Pseudo-ground-truth}
\end{subfigure}
\begin{subfigure}[t]{0.32\textwidth}
\includegraphics[width=1\textwidth]{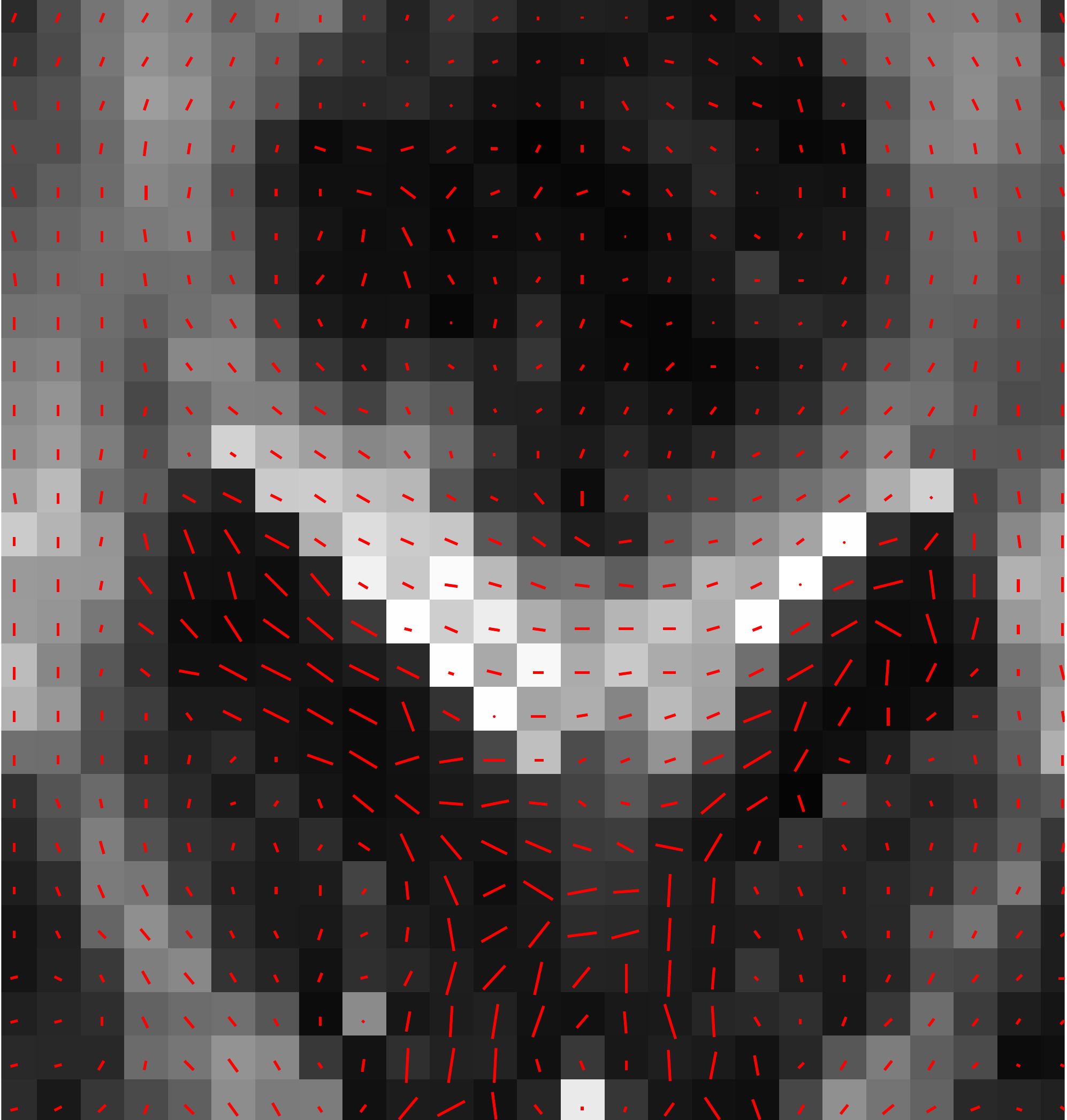}
\subcaption{Linear $L^2$, discrepancy principle}
\end{subfigure}
\begin{subfigure}[t]{0.32\textwidth}
\includegraphics[width=1\textwidth]{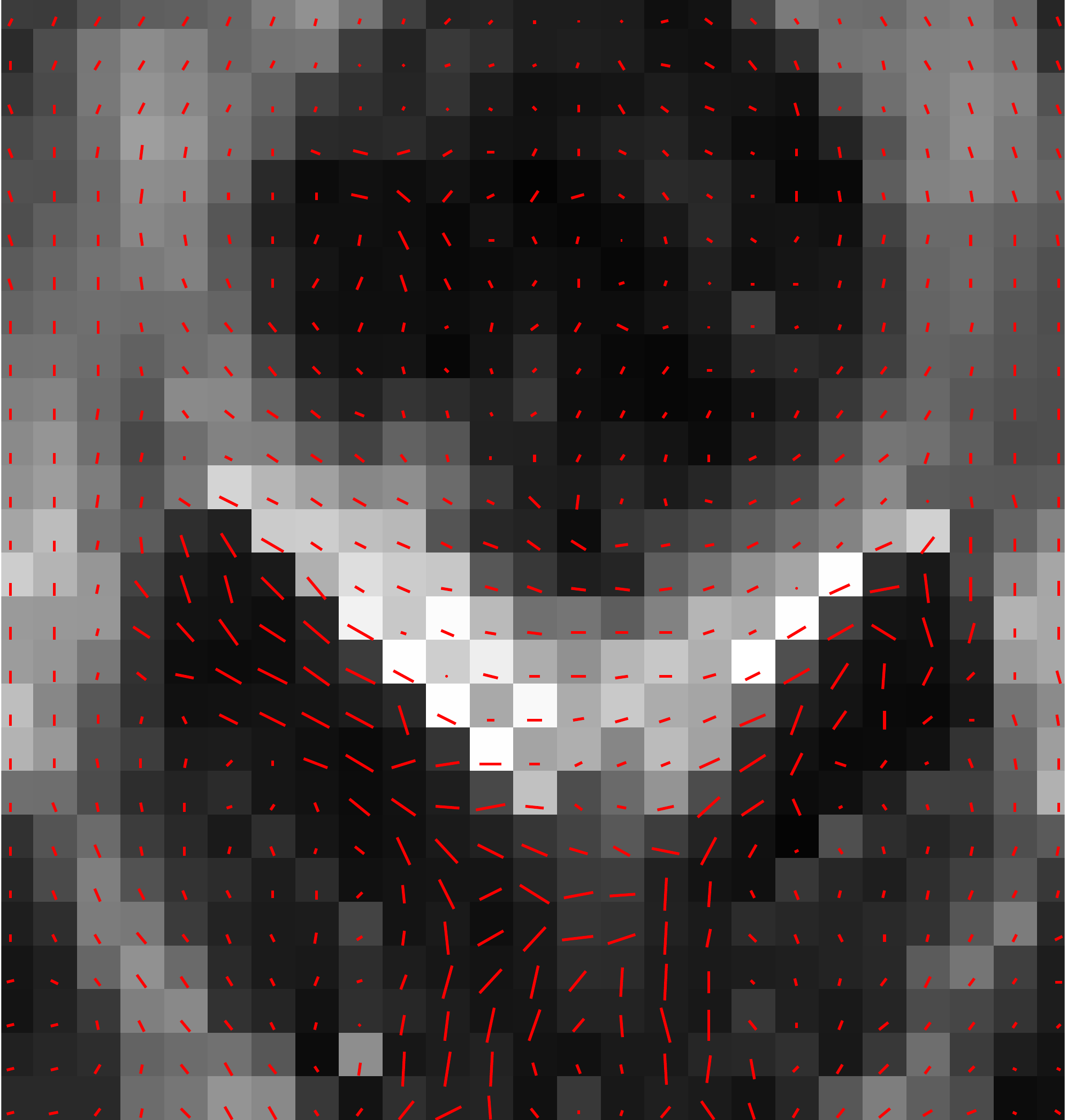}
\subcaption{Linear $L^2$, error-optimal}
\end{subfigure}
\begin{subfigure}[t]{0.32\textwidth}
\includegraphics[width=1\textwidth]{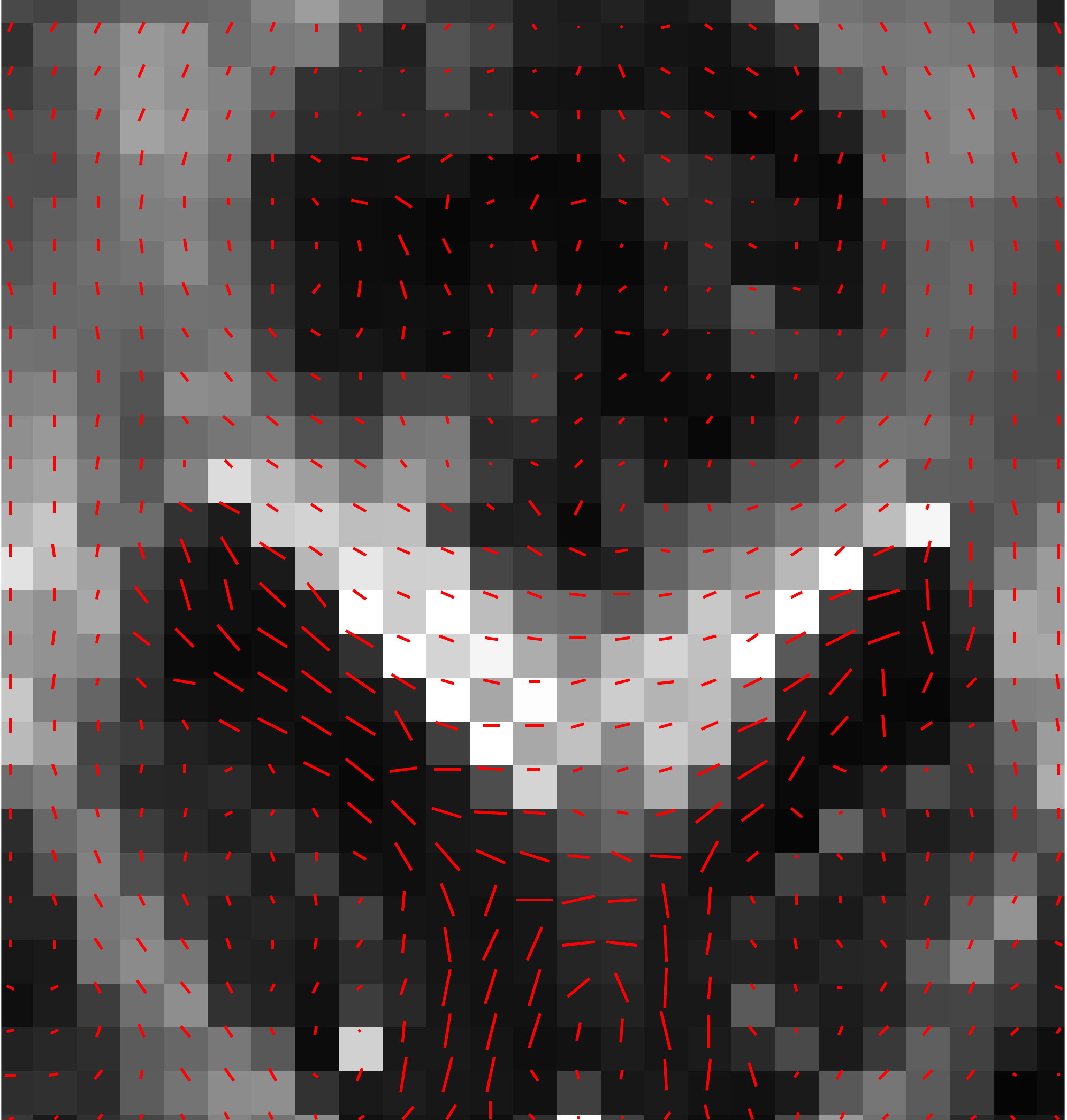}
\subcaption{Non-linear $L^2$, discrepancy principle}
\end{subfigure}
\begin{subfigure}[t]{0.32\textwidth}
\includegraphics[width=1\textwidth]{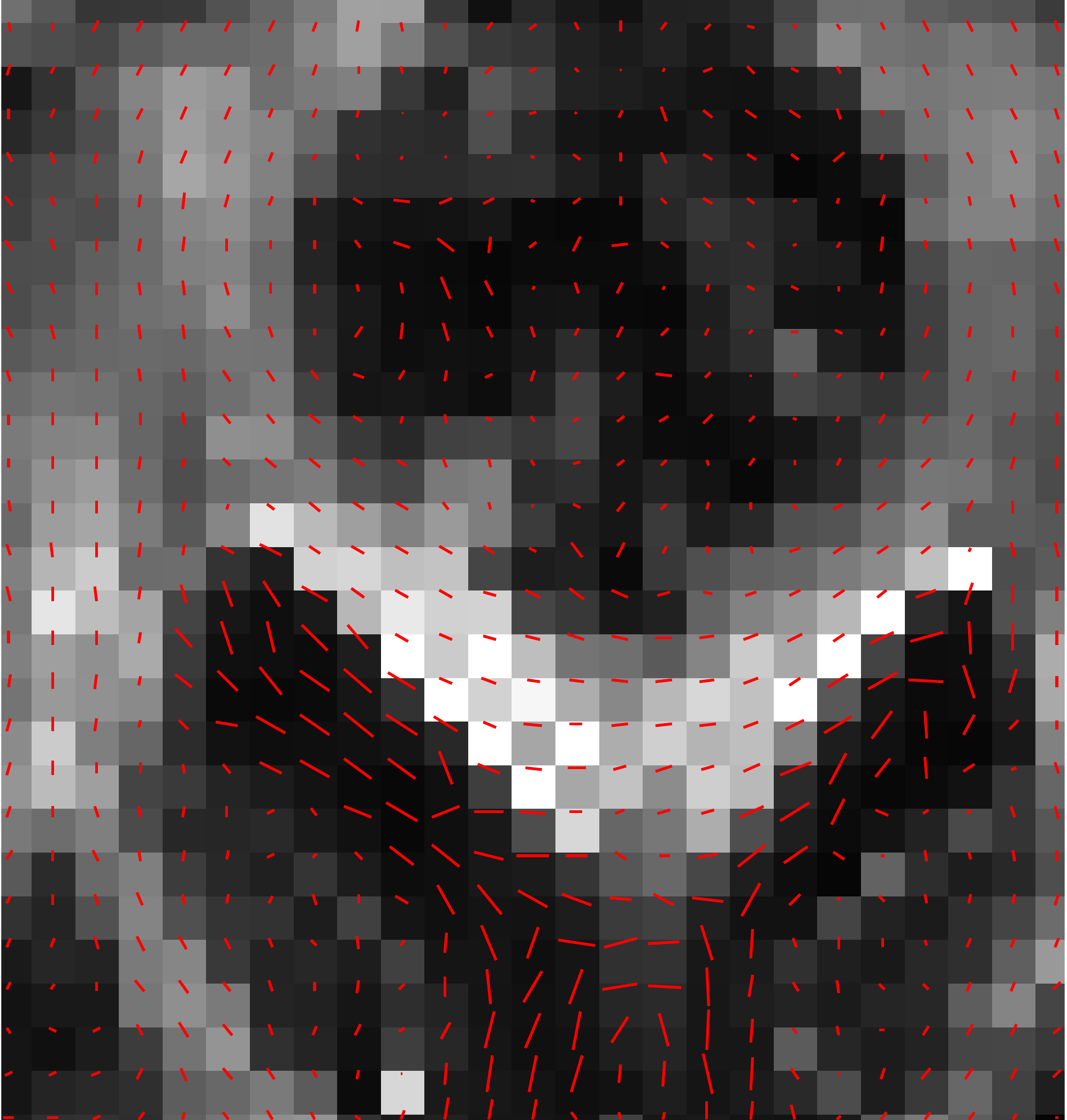}
\subcaption{Non-linear $L^2$, error-optimal}
\end{subfigure}
\begin{subfigure}[t]{0.32\textwidth}
\includegraphics[width=1\textwidth]{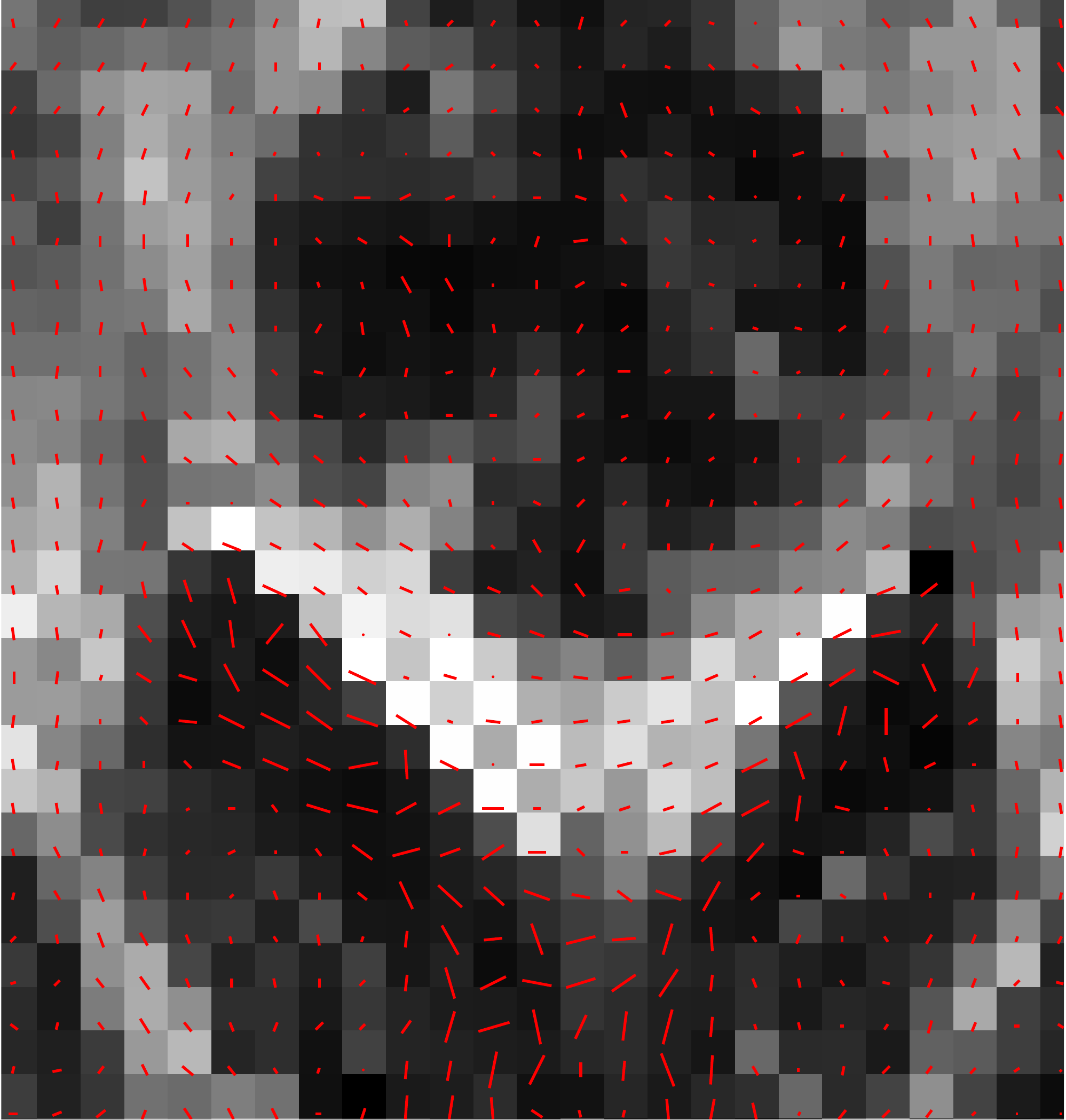}
\subcaption{Constrained, 95\% confidence intervals}
\end{subfigure}

\begin{minipage}[c]{0.8\textwidth}
\caption{Fractional anisotropy of the \emph{corpus callosum} in greyscale superimposed by principal eigenvector.
Legend on the right indicates the greyscale intensities of the fractional anisotropy.}
\label{fig:invivo-fa-pev}
\end{minipage}%
\begin{minipage}[t]{0.2\textwidth}
\flushright%
\setlength{\w}{\textwidth}
\tikzexternaldisable
\begin{tikzpicture}
        \scriptsize

        \shade [shading=axis, left color=black, right color=white]
                (-0.3\w,0.05) rectangle (0.3\w,0.1\w);

        \draw[->] (-0.3\w,0)
                  node[below] {0} --
                  node[below] {$\FA$} (0.3\w,0)
                  node[below] {1};

        \def\mp#1#2{\begin{minipage}{#1}
                #2
                \legendstyle
                Greyscale-coding of the fractional anisotropy.
                Principal eigenvector is drawn in red.
        \end{minipage}}

\end{tikzpicture}
\tikzexternalenable
\null
\vfill
\end{minipage}
\end{figure}

\def\SQl{112}
\def\SQb{130}
\def\SQr{176}
\def\SQt{194}
\newlength{\scf}
\newlength{\ww}
\def\shadowshift{0pt}
\newcommand{\drawzoomarea}{
    \draw[line width=1.5,dashed,color=red] 
        (\SQl\scf, \SQb\scf) rectangle (\SQr\scf, \SQt\scf);
}

\begin{figure}[!htp]
    \newcommand{\includegraphicspip}[3][]{%
        \setlength{\ww}{#2}%
        \setlength{\scf}{0.0039\ww}%
        \begin{tikzpicture}
        \pgftext[at=\pgfpoint{0}{0},left,bottom]{
            \includegraphics[width=\ww]{{#3}}%
        }%
        \pgftext[at=\pgfpoint{0.59\ww}{0.01\ww},left,bottom]{%
            \includegraphics[width=0.4\ww,bb=112 130 176 194,clip]{#3}%
        }%
        \draw[line width=1.5, color=red] (0.59\ww, 0.01\ww) rectangle (0.99\ww, 0.41\ww);
        #1%
        \end{tikzpicture}%
    }
\centering%
\begin{subfigure}[t]{0.32\textwidth}
\includegraphicspip[\drawzoomarea]{\textwidth}{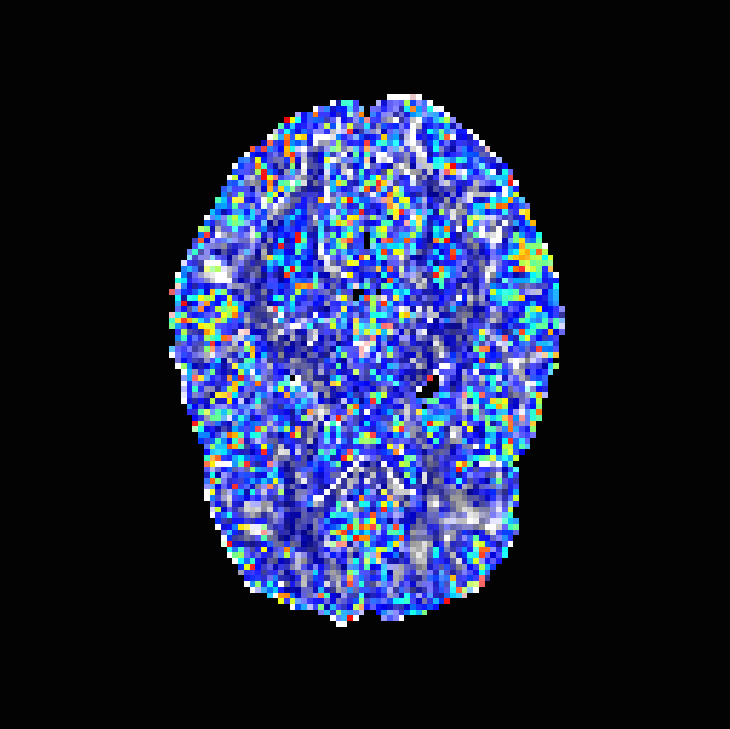}
\subcaption{Linear $L^2$, discrepancy principle}
\end{subfigure}
\begin{subfigure}[t]{0.32\textwidth}
\includegraphicspip{\textwidth}{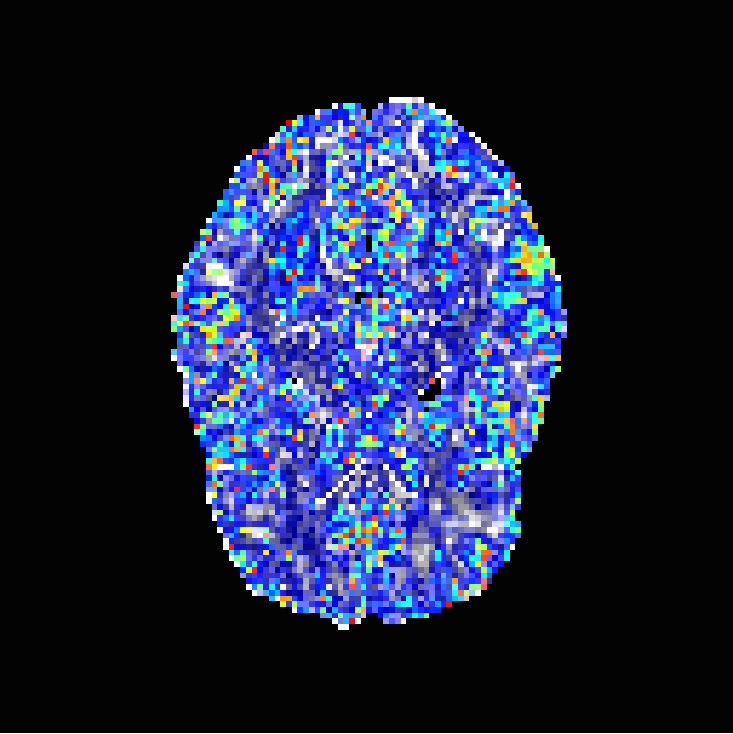}
\subcaption{Linear $L^2$, error-optimal}
\end{subfigure}
\begin{subfigure}[t]{0.32\textwidth}
\includegraphicspip{\textwidth}{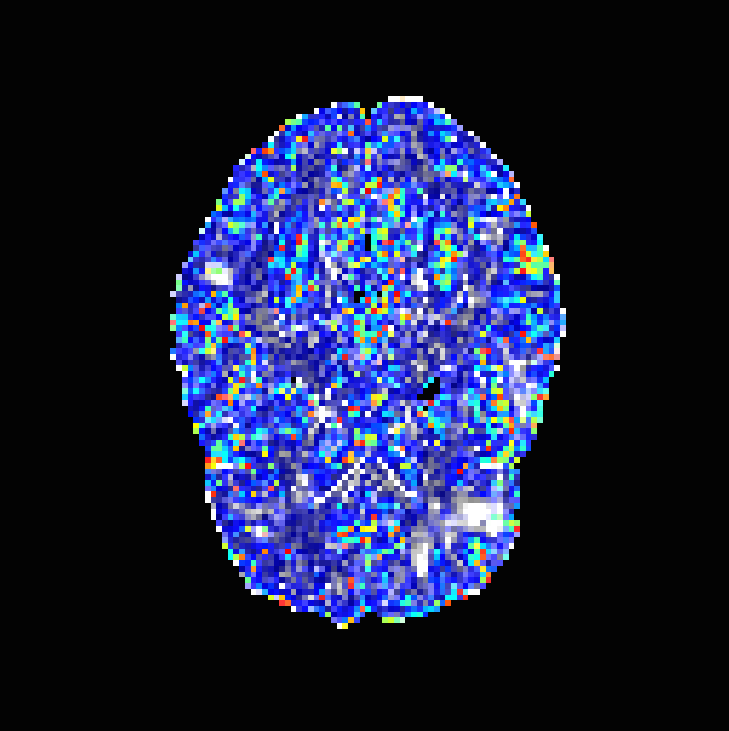}
\subcaption{Non-linear $L^2$, discrepancy principle}
\end{subfigure}
\\
\begin{subfigure}[t]{0.32\textwidth}
\includegraphicspip{\textwidth}{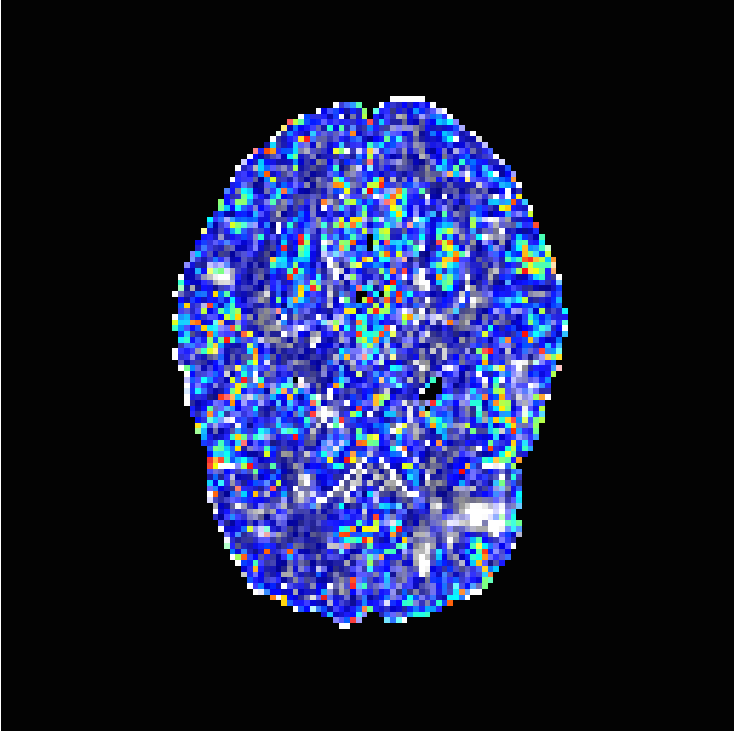}
\subcaption{Non-linear $L^2$, error-optimal}
\end{subfigure}
\begin{subfigure}[t]{0.32\textwidth}
\includegraphicspip{\textwidth}{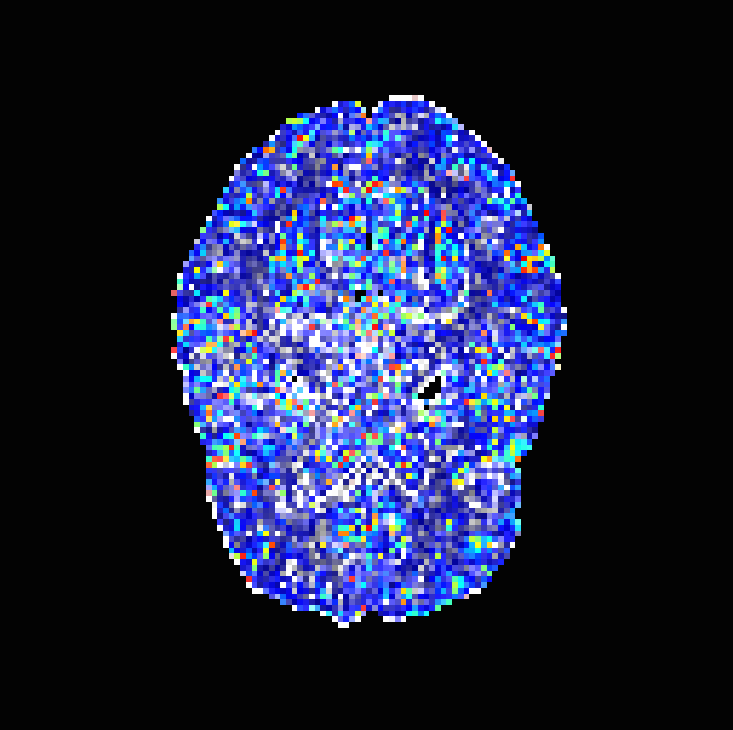}
\subcaption{Constrained, 95\% confidence intervals}
\end{subfigure}
\begin{subfigure}[b]{0.32\textwidth}
\setlength{\w}{0.7\textwidth}
\centering
\input{figures/legend2-errplot.tikz}
\vspace{.8cm}
\end{subfigure}
\caption{Colour-coded errors of fractional anisotropy and principal eigenvector for the computations on the synthetic test data. Legend on the right indicates the
colour-coding of errors between
$u$ and $g_0$ as functions of the principal eigenvector angle error
\mbox{$\theta=\inv \cos(\iprod{\hat v_u}{\hat v_{g_0}})$} in terms of the hue, and the fractional anisotropy error
\mbox{$e_\FA=\abs{\FA_u -\FA_{g_\mathrm{0}}}$} in terms of whiteness.
}
\label{fig:invivo-fa-peangle}
\end{figure}

\begin{figure}[!htp]
\begin{subfigure}[t]{0.32\textwidth}
\includegraphics[width=1\textwidth]{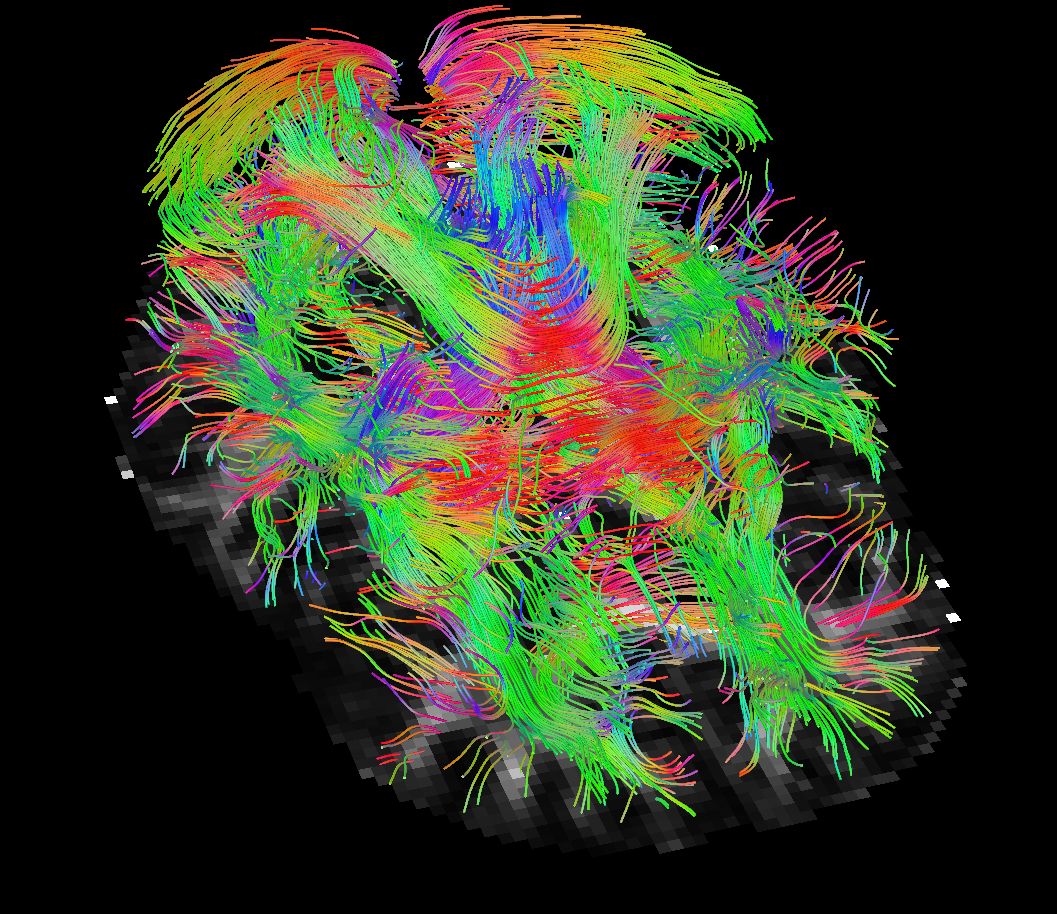}
\subcaption{Pseudo-ground-truth}
\end{subfigure}
\begin{subfigure}[t]{0.32\textwidth}
\includegraphics[width=1\textwidth]{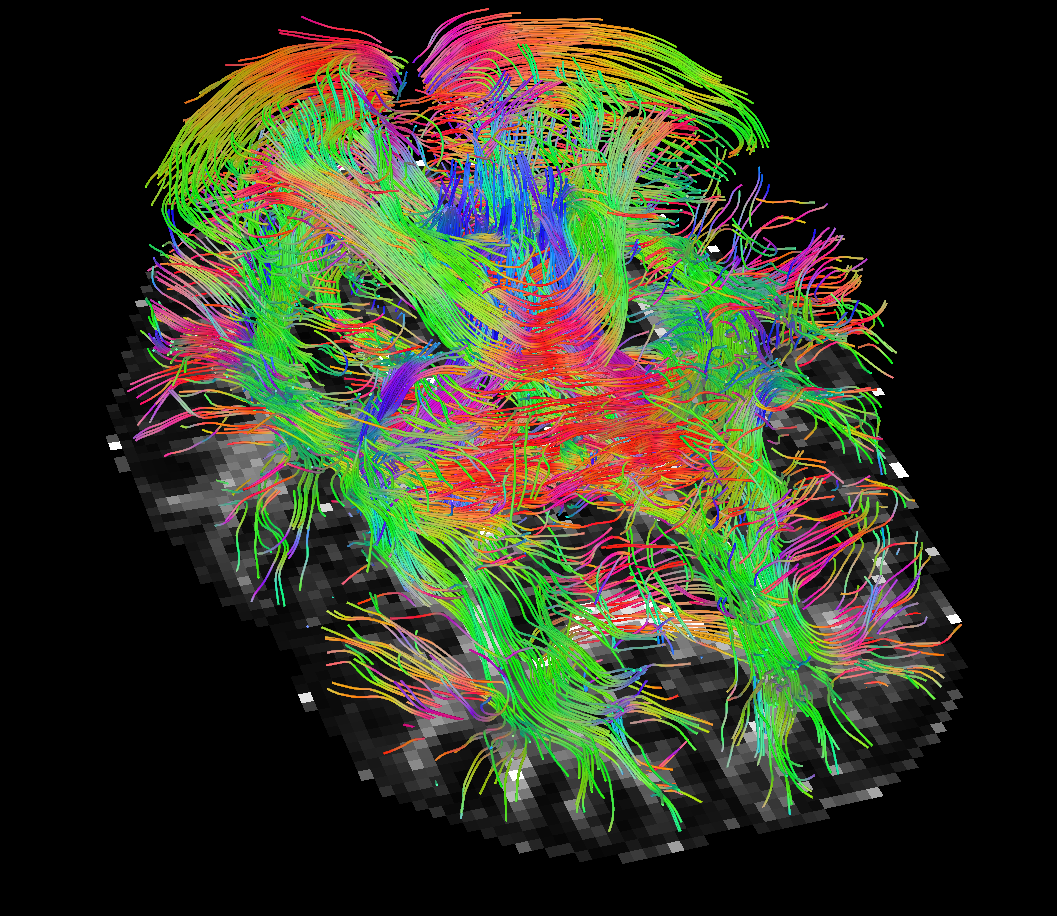}
\subcaption{Regression result}
\end{subfigure}
\begin{subfigure}[t]{0.32\textwidth}
\includegraphics[width=1\textwidth]{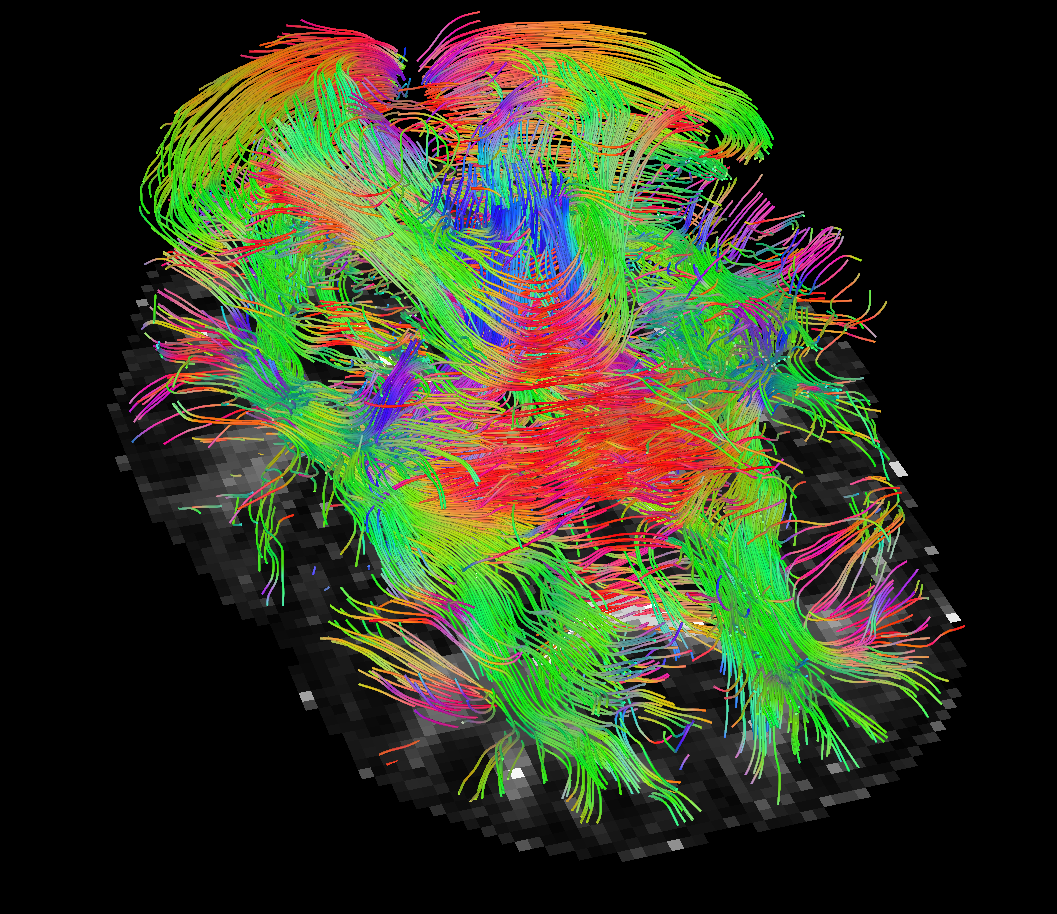}
\subcaption{Linear $L^2$, discrepancy principle}
\end{subfigure}
\begin{subfigure}[t]{0.32\textwidth}
\includegraphics[width=1\textwidth]{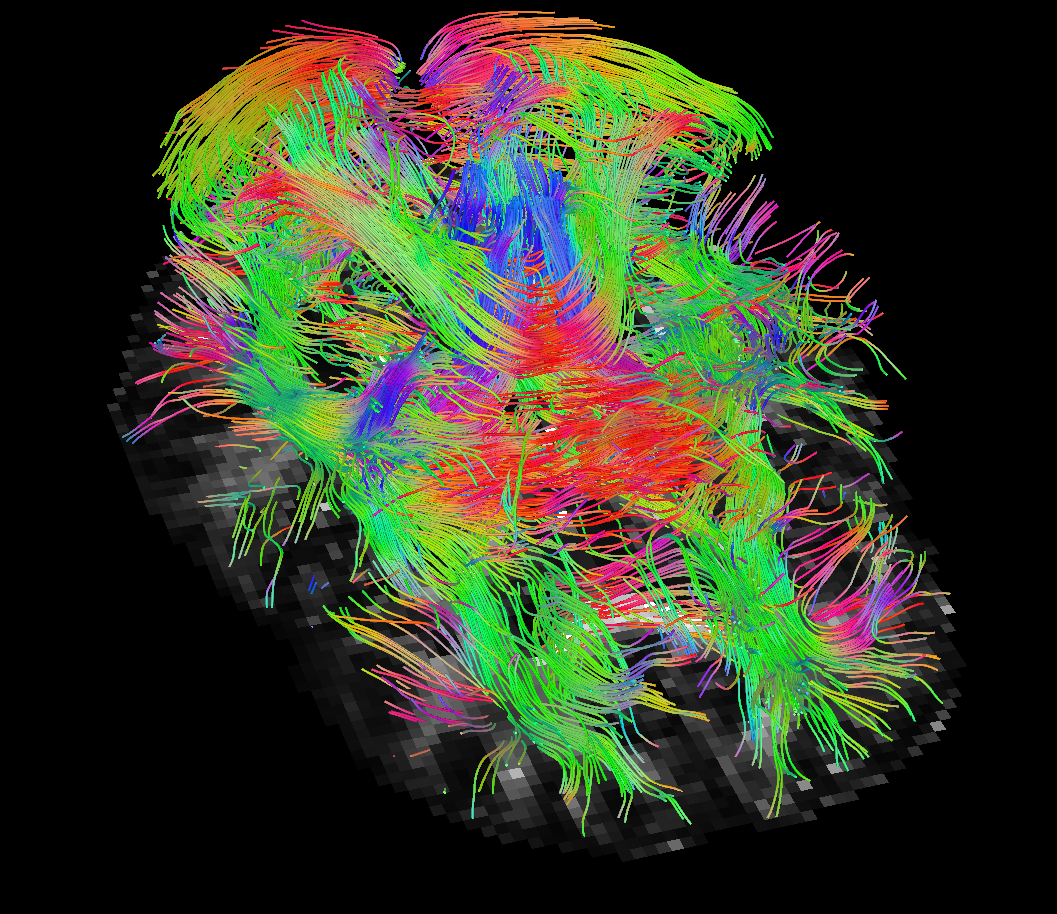}
\subcaption{Non-linear $L^2$, discrepancy principle}
\end{subfigure}
\begin{subfigure}[t]{0.32\textwidth}
\includegraphics[width=1\textwidth]{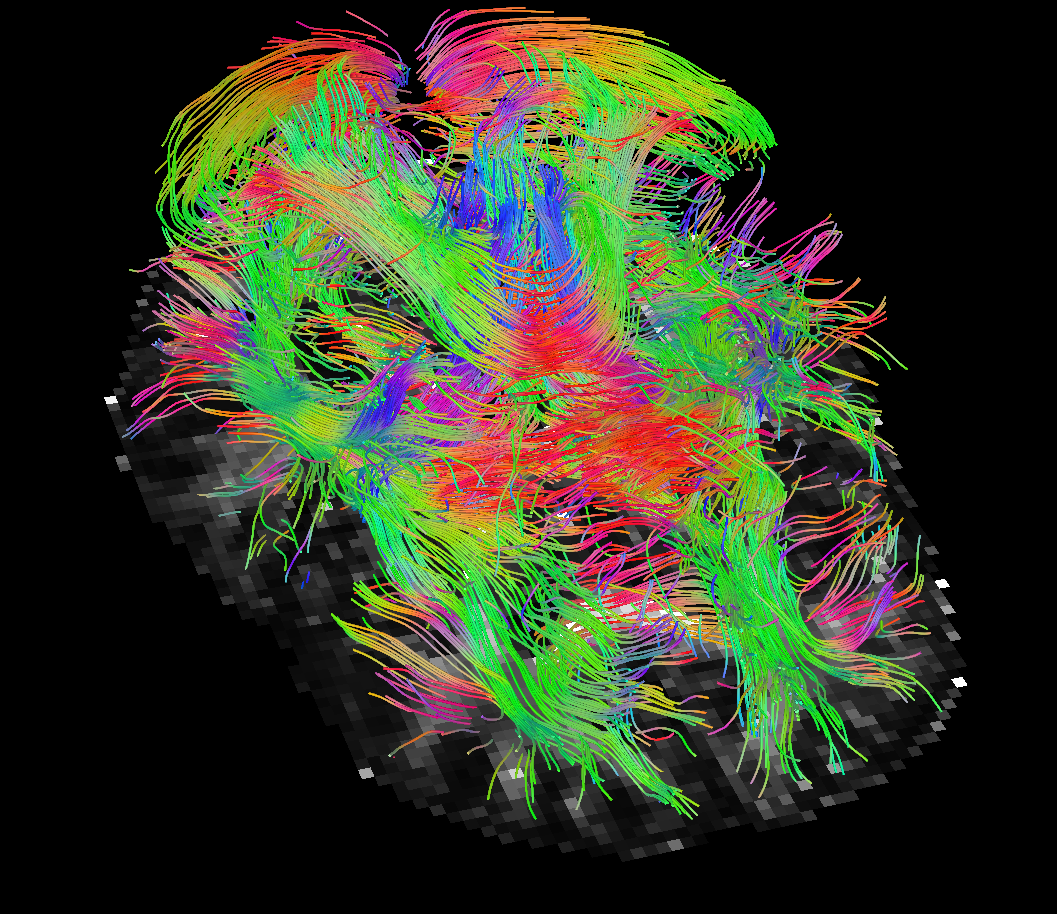}
\subcaption{Non-linear $L^2$, error-optimal}
\end{subfigure}
\begin{subfigure}[t]{0.32\textwidth}
\includegraphics[width=1\textwidth]{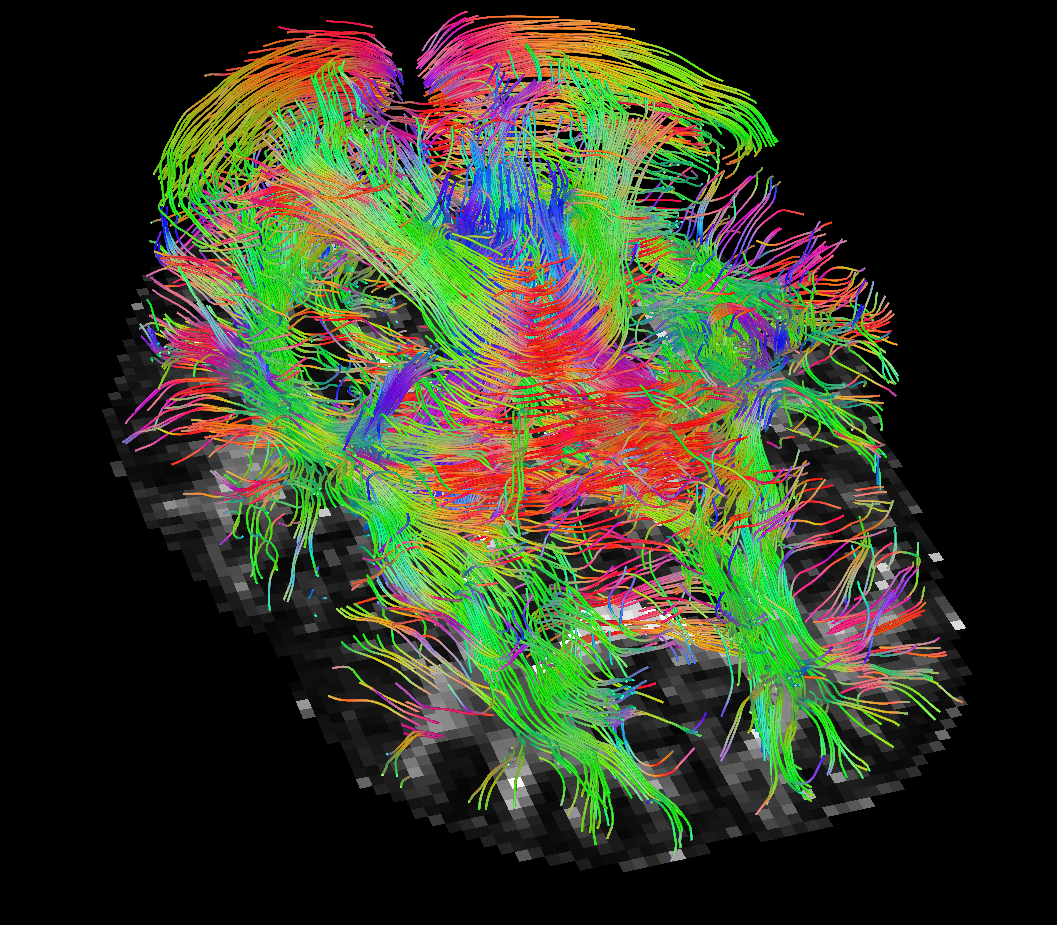}
\subcaption{Constrained, 95\% confidence intervals}
\end{subfigure}
\caption{Tractography visualization results on the in vivo brain data. (a) Pseudo-ground-truth plot. (b) regression result. (c)--(f) Plot of a slice of the solution for $L^2$, non-linear $L^2$ and constrained problem approach.
}
\label{fig:invivo-trac}
\end{figure}
This time, the linear $L^2$ approach~\eqref{eq:dti-recons-lin1} has best overall reconstruction (Frobenius PSNR), while the nonlinear $L^2$ approach~\eqref{eq:dti-recons-nl} has clearly the best principal eigenvector angle reconstruction besides the regression, which does not seem entirely reliable regarding our regression-based pseudo-ground-truth. The constraints based approach~\eqref{eq:dti-recons-constr}, with 95\% confidence intervals is, however, not far behind in terms of numbers. More detailed study of the \emph{corpus callosum} in Figure \ref{fig:invivo-fa-peangle} (small picture in picture) and Figure \ref{fig:invivo-fa-pev} however indicates a better reconstruction of this important region by the nonlinear approach. The constrained approach has some very short vectors there in the white region.
Naturally, however, these results on the in vivo data should be taken with a grain of salt, as we have only a somewhat unreliable pseudo-ground-truth available for comparison purposes.

\subsection{Conclusions from the numerical experiments}

Our conclusion is that the error bounds based approach is a feasible alternative to standard modelling with incorrect Gaussian assumptions. It can produce good reconstructions, although the non-linear $L^2$ approach of \cite{tuomov-nlpdhgm} is possibly slightly more reliable. The latter does, however, in principle depend on a good initialisation of the optimisation method, unlike the convex bounds based approach.

Further theoretical work will be undertaken to extend the partial-order-based approach to modelling errors in linear operators to the non-lattice case of the semidefinite partial order for symmetric matrices, which will allow us to consider problems of diffusion MRI with errors in the forward operator.

It shall also be investigated whether the error bounds approach needs to be combined with an alternative, novel, regulariser that would ameliorate the fractional anisotropy errors that the approach exhibits. It is important to note, however, that from the practical point of view, of using the reconstruction tensor field for basic tractography methods based solely on principal eigenvectors, these are not that critical. As pointed out by one of the reviewers, the situation could differ with more recent geodesic tractography methods \cite{hao2011adaptive,schultz2014novel,fuster2016adjugate} employing the full tensor.
We provide basic principal eigenvector tractography results for reference in Figure \ref{fig:invivo-trac}, without attempting to extensively interpret the results. It suffices to say that the results look comparable. With this in sight, the error bounds approach produces a very good reconstruction of the direction of the principal eigenvectors, although we saw some problems with the magnitude within the \emph{corpus callosum}.

\section*{Acknowledgements}

While at the Center for Mathematical Modelling of the Escuela Polit\'ecnica Nacional in Quito, Ecuador, T.~Valkonen has been supported by a Prometeo scholarship of the Senescyt (Ecuadorian Ministry of Science, Technology, Education, and Innovation). In Cambridge, T.~Valkonen has been supported by the {EPSRC} grants Nr.~EP/J009539/1 ``Sparse \& Higher-order Image Restoration'', and Nr.~EP/M00483X/1 ``Efficient computational tools for inverse imaging problems''. 

A. Gorokh and Y. Korolev are grateful to the RFBR (Russian Foundation for Basic Research) for partial financial support (projects 14-01-31173 and 14-01-91151).

The authors would further like to thank Karl Koschutnig for the in vivo data set, Kristian Bredies for scripts used to generate the tractography images, and Florian Knoll for many inspiring discussions.

\section*{A data statement for the EPSRC}

The MRI scans used for this publication have been used by the courtesy of the producer.
Data that is similar for all intents and purposes, is widespread, and can be easily produced by making a measurement of a human subject with an MRI scanner.
Our source codes are archived at \url{https://www.repository.cam.ac.uk/handle/1810/253422}.

\appendix

\section*{Appendix: Notation and techniques}
\label{sec:notation}

We recall some basic, not completely standard, mathematical notation and concepts in this appendix. We begin with partially ordered vector spaces, following the book~\cite{Schaefer}.  Then we proceed to tensor calculus and \term{tensor fields of bounded variation} and of \term{bounded deformation}. These are also covered in more detail for the diffusion tensor imaging application in \cite{tuomov-dtireg}.

\subsection*{Banach lattices}
A linear space $X$, endowed with a partial order relation $\leqslant$ is called an ordered vector space if the partial order agrees with linear operations in the following way:
\begin{equation*}
\begin{aligned}
&x \leqslant y &&\implies &&x + z \leqslant y+z \quad &&\forall x,y,z \in X, \\
&x \leqslant y &&\implies  &&\lambda x \leqslant \lambda y \quad &&\forall x,y \in X \text{ and } \lambda \in \mathbb R_+.
\end{aligned}
\end{equation*}
An ordered vector space is called a vector lattice if each pair of elements $x,y \in X$ have a supremum $x \vee y \in X$ and infimum $x \wedge y \in X$. Supremum of two elements $x,y$ of a Banach lattice $X$ is the element $z   = x \vee y$ with the following properties: $z \geqslant x, \,\, z \geqslant y$ and $\forall \tilde z \in X$ such that
$\tilde z \geqslant x$ and $\tilde z \geqslant y$ we have $\tilde z \geqslant z$.

For any $x \in X$, the element $x_+ = x \vee 0$ is called its positive part, the element $x_- = (-x) \vee 0 = (-x)_+$ is called its negative part, the element $|x| = x_+ + x_-$ is its absolute value. The equalities $x = x_+ - x_-$ and $|x| = x \vee (-x)$ hold  for any $x \in X$.

It is obvious that suprema and infima exist for any finite number of elements of a vector lattice. A vector lattice $X$ is said to be order complete if \emph{any} bounded from above set in $X$ has a supremum.

Let $X$ and $Y$ be ordered vector spaces. A linear operator $U \colon X \to Y$ is called positive, if $x \geqslant_X 0$ implies $Ux \geqslant_Y 0$. An operator $U$ is called regular, if it can be written as $U = U_1 - U_2$, where $U_1$ and $U_2$ are positive operators.

Denote the linear space of all regular operators $X \to Y$ by $L^\sim(X,Y)$. A partial order can be introduced in $L^\sim(X,Y)$ in the following way: $U_1 \geqslant U_2$, if $U_1 - U_2$ is a positive operator. If $X$ and $Y$ are vector lattices and $Y$ is order complete, then $L^\sim(X,Y)$ is also an order complete vector lattice.

A norm $\| \cdot \|$ defined in a vector lattice $X$ is called monotone if $|x| \leqslant |y|$ implies  $\|x\| \leqslant \|y\|$. A vector lattice endowed with a monotone norm is called a Banach lattice if it is norm complete. If $X$ and $Y$ are Banach lattices, then all operators in $L^\sim(X,Y)$ are continuous.

Let us list some examples of Banach lattices. The space of continuous functions $C(\Omega)$, where $\Omega \subset \mathbb R^n$, is a Banach lattice under the natural pointwise ordering: $f \geqslant_C g$ if and only if $f(x) \geqslant g(x)$ for all $x \in \Omega$. The spaces $L^p(\Omega)$, $1 \leqslant p \leqslant \infty$, are also Banach lattices under the following partial ordering: $f \geqslant_{L^p} g$ if and only if $f(x) \geqslant g(x)$ almost everywhere in $\Omega$. With this partial order, $L^p(\Omega)$, $1 \leqslant p \leqslant \infty$, are order complete Banach lattices. The Banach lattice of continuous functions $C(\Omega)$ is not order complete.

\subsection*{Tensors in the Euclidean setting}

We call a $k$-linear mapping $A: \R^m \times \cdots \times \R^m \to \R$ a $k$-tensor, denoted $A \in \Tensor^k(\R^m)$. This is a simplification from the full differential-geometric definition, sufficient for our finite-dimensional setting. We say that $A$ is symmetric, denoted $A \in \Sym^k(\R^m)$, if it satisfies for any permutation $\pi$ of $\{1,\ldots,k\}$ that
\[
    A(c_{\pi 1},\ldots,c_{\pi k})=A(c_1,\ldots,c_k).
\]

With $e_1,\ldots,e_m$ the standard basis of $\R^m$, we define on $\Tensor^k(\R^m)$ the inner product
\[
    \iprod{A}{B} \defeq \sum_{p \in \{1,\ldots,m\}^k}
            A(e_{p_1},\ldots,e_{p_k})
            B(e_{p_1},\ldots,e_{p_k}),
\]
and the Frobenius norm
\[
    \norm{A}_F \defeq \sqrt{\iprod{A}{A}}.
\]
The Frobenius norm is rotationally invariant in a sense crucial for DTI. We refer to \cite{tuomov-dtireg} for a detailed discussion of this, as well of alternative rotationally invariant norms.

\begin{example}[Vectors]
Vectors $A \in \R^m$ are of course symmetric $1$-tensors,
The inner product is the usual inner product in $\R^m$,
and the Frobenius norm is the two-norm, $\norm{A}_F=\norm{A}_2$.
\end{example}

\begin{example}[Matrices]
Matrices are $2$-tensors: $A(x, y) = \iprod{A x}{y}$, while symmetric matrices $A=A^T$ are symmetric $2$-tensors.
The inner product is $\iprod{A}{B} =\sum_{i,j} A_{ij}B_{ij}$  and $\norm{A}_F$ is the matrix Frobenius norm.

We use the notation $A \geq 0$ for positive-semidefinite matrices $A$. One can verify that this relation indeed defines a partial order in the space of symmetric matrices:
\begin{equation}\label{Loewner_order}
A \geq B \quad \text{ iff  $A-B$ is positive semidefinite}.
\end{equation}
With this partial order, the space of all symmetric matrices becomes an ordered vector space, but not a vector lattice. However, it enjoys some properties similar to those of vector lattices: for example, any directed upwards subset\footnote{Recall that an indexed subset $\{ x_\tau \colon \tau \in \{\tau\} \}$ of an ordered vector space $X$ is called directed upwards if for any pair $\tau_1,\,\tau_2 \in \{\tau\}$ there exists $\tau_3 \in \{\tau\}$ such that $x_{\tau_3} \geqslant x_{\tau_1}$ and $x_{\tau_3} \geqslant x_{\tau_2}$.} has a supremum~\cite[Ch.8]{Luxemburg_Zaanen}.
\end{example}

\subsection*{Symmetric tensor fields of bounded deformation}
\label{Symmetric tensor fields of bounded deformation}
Let $u: \Omega \to \Sym^k(\R^m)$ for a domain $\Omega \subset \R^m$. We set
\[
    \norm{u}_{F,p}
    \defeq
    \Bigl(\int_\Omega \norm{u(x)}_F^p \d x\Bigr)^{1/p}\hspace{-3ex},
    \hspace{4ex} (p \in [1,\infty)),
\]
and
\[
    \norm{u}_{F,\infty}
    \defeq
        \esssup_{x \in \Omega} \norm{u(x)}_F,
\]
The spaces $L^p(\Omega; \Sym^{k}(\R^m))$ are defined in the natural way using these norms, and clearly satisfy all the usual properties of $L^p$ spaces.

In the particular case of matrices ($k=2$), partial order can be introduced in the space $L^p(\Omega; \Sym^{2}(\R^m))$ in the following way:
\begin{equation}\label{our_order}
u \geqq v \quad \text{ iff  $u(x) \ge v(x)$ almost everywhere in $\Omega$}.
\end{equation}
Recall that the symbol $\geq$ stands for the positive semidefinite order~(\ref{Loewner_order}) in the space of symmetric matrices. Since the positive semidefinite order is not a lattice, neither is the order~(\ref{our_order}).

If $u \in C^1(\Omega; \Sym^k(\R^m))$, $k\ge 1$, we define by contraction the divergence $\divergence u \in C(\Omega; \Sym^{k-1}(\R^m)$ as
\begin{equation}
    \label{eq:divergence-def}
    [\divergence u(\cdot)](e_{i_2},\ldots,e_{i_k})
    \defeq
    \sum_{i_1=1}^m \iprod{e_{i_1}}{\grad u(\cdot)(e_{i_1},\ldots,e_{i_k})}.
\end{equation}
It is easily verified that $\divergence u(x)$ is indeed symmetric. Given a tensor field $u \in L^1(\Omega; \Tensor^{k}(\R^m))$ we then define the \emph{symmetrised distributional gradient} $E u \in [C_c^\infty(\Omega; \Sym^{k+1}(\R^m))]^*$ by
\[
    E u(\varphi) \defeq -\int_\Omega \iprod{u(x)}{\divergence \varphi(x)} \d x,
    \quad
    (\varphi \in C_c^\infty(\Omega; \Sym^{k+1}(\R^m))).
\]

With these notions at hand, we now define the spaces of \term{symmetric
tensor fields of bounded deformation} as (see also \cite{tuomov-dtireg,bredies2014symmetric})
\begin{align}
    \notag
    \BDspace(\Omega; \Sym^{k}(\R^m)) & \defeq
    \Bigl\{u \in L^1(\Omega; \Sym^{k}(\R^m)) \Bigm|
        \sup\nolimits_{\varphi \in \BALL^{k+1}_{F,\sym}} Eu(\varphi) < \infty \Bigr\},
\end{align}
where
\[
    \BALL^{k}_{F,\sym} \defeq \{ \varphi \in C_c^\infty(\Omega; \Sym^{k}(\R^m)) \mid \norm{\varphi}_{F,\infty} \le 1 \}.
\]
For $u \in \BDspace(\Omega; \Sym^{k}(\R^m))$, the symmetrised gradient $Eu$ is a Radon measure, $Eu \in \Meas(\Omega; \Sym^{k+1}(\R^m))$. For the proof of this fact we refer to \cite[\textsection 4.1.5]{federer1969gmt}.

\begin{example}
    The space $\BDspace(\Omega; \Sym^{0}(\R^m))$ agrees with the space $\BVspace(\Omega)$ of functions of bounded variation.
    The space $\BDspace(\Omega; \Sym^{1}(\R^m))=\BDspace(\Omega)$ is the space of \term{functions of bounded deformation} of \cite{temam1985mpp}.
\end{example}

\bibliographystyle{amsplnat-blkurl}
\bibliography{abbrevs,constr-dti}


\end{document}